\def\year{2021}\relax
\documentclass[letterpaper]{article} %
\usepackage{aaai21}  %
\usepackage{times}  %
\usepackage{helvet} %
\usepackage{courier}  %
\usepackage[hyphens]{url}  %
\usepackage{graphicx} %
\usepackage{natbib}  %
\usepackage{caption} %
\newcommand{\blue}[1]{{#1}}
\usepackage{pdfpages}

\usepackage[utf8]{inputenc} %
\usepackage{booktabs}       %
\usepackage{amsfonts}       %
\usepackage{nicefrac}       %
\usepackage{microtype}      %

\usepackage{comment} 
\usepackage{verbatim}
\usepackage{listings}

\usepackage{amsmath,amssymb,amsthm}
\usepackage{colonequals}
\usepackage{multirow}
\usepackage{algpseudocode,algorithm,algorithmicx} 
\usepackage{tikz}
\usepackage{xparse}
\usepackage{subcaption}
\ExplSyntaxOn
\DeclareExpandableDocumentCommand{\floor}{m}
{
	\fp_eval:n { floor ( #1 ) }
}
\ExplSyntaxOff

\lstdefinestyle{mystyle}{
	breakatwhitespace=false,         
	breaklines=true,                 
	captionpos=b,                    
	keepspaces=true,                 
	numbers=none,                    
	numbersep=5pt,                  
	showspaces=false,                
	showstringspaces=false,
	showtabs=false,                  
	tabsize=2
}
\usetikzlibrary{decorations.pathreplacing,calligraphy, calc}

\usepackage{array}
\newcommand{\PreserveBackslash}[1]{\let\temp=\\#1\let\\=\temp}
\newcolumntype{C}[1]{>{\PreserveBackslash\centering}p{#1}}
\newcolumntype{R}[1]{>{\PreserveBackslash\raggedleft}p{#1}}
\newcolumntype{L}[1]{>{\PreserveBackslash\raggedright}p{#1}}

\DeclareMathOperator{\conv}{conv}

\DeclareMathOperator{\vect}{vec}
\DeclareMathOperator{\nz}{nz}
\DeclareMathOperator{\midd}{mid}
\DeclareMathOperator{\set}{set}
\DeclareMathOperator{\abs}{abs}

\newtheorem{theorem}{Theorem}%
\newtheorem{lemma}{Lemma}
\newtheorem{proposition}{Proposition}

\theoremstyle{definition}
\newtheorem{definition}{Definition}
\newtheorem*{setting*}{Setting}

\theoremstyle{remark}
\newtheorem{remark}{Remark}
\theoremstyle{plain}
\newtheorem{assumption}{Assumption}

\newcommand{\ceq}{\colonequals}

\newcommand{\Mod}[1]{\ (\mathrm{mod}\ #1)}
\newcommand{\norms}{the $\ell_1$ norm and $\ell_\infty$ norm }
\newcommand{\cifar}{CIFAR-10}

\newcommand{\cl}{\mathrm{cl}}
\newcommand{\cq}{\colonequals}
\newcommand{\bv}[1]{\boldsymbol{#1}} %
\newcommand{\mc}[1]{\mathcal{#1}}
\newcommand{\mv}[2]{\mathcal{#1}^{(#2)}}
\newcommand{\ttt}[1]{\text{\textnormal{#1}}}
\newcommand{\loss}{\mathcal{L}}
\newcommand{\ball}{\mathrm{B}}
\newcommand{\sgd}{\operatorname{SGD}}

\newcommand{\ccite}[1]{\cite{#1}}
\newcommand{\ccitet}[1]{\citet{#1}}

\title{Large Norms of CNN Layers Do Not Hurt Adversarial Robustness}
\author{
    Youwei Liang\textsuperscript{\rm 1}, Dong Huang\textsuperscript{\rm 1,2}\thanks{Corresponding author.}%
    \\
}
\affiliations{
    \textsuperscript{\rm 1}
    College of Mathematics and Informatics, South China Agricultural University, Guangzhou, China\\
    \textsuperscript{\rm 2}
    Pazhou Lab, Guangzhou, China\\
    liangyouwei1@gmail.com, huangdonghere@gmail.com
}

\begin{document}

\maketitle

\begin{abstract}
Since the Lipschitz properties of convolutional neural networks (CNNs) are widely considered to be related to adversarial robustness, we theoretically characterize the $\ell_1$ norm and $\ell_\infty$ norm of 2D multi-channel convolutional layers and provide efficient methods to compute the exact $\ell_1$ norm and $\ell_\infty$ norm. Based on our theorem, we propose a novel regularization method termed norm decay, which can effectively reduce the norms of convolutional layers and fully-connected layers. Experiments show that norm-regularization methods, including norm decay, weight decay, and singular value clipping, can improve generalization of CNNs. However, they can slightly hurt adversarial robustness. Observing this unexpected phenomenon, we compute the norms of layers in the CNNs trained with three different adversarial training frameworks and surprisingly find that adversarially robust CNNs have comparable or even larger layer norms than their non-adversarially robust counterparts. 
Furthermore, we prove that under a mild assumption, adversarially robust classifiers can be achieved using neural networks, and an adversarially robust neural network can have an arbitrarily large Lipschitz constant. For this reason, enforcing small norms on CNN layers may be neither necessary nor effective in achieving adversarial robustness. The code is available at \url{https://github.com/youweiliang/norm_robustness}. 
\end{abstract}

\section{Introduction}
Convolutional neural networks (CNNs) have enjoyed great success in computer vision \cite{lecun2015deep, goodfellow2016deep}. However, many have found that CNNs are vulnerable to adversarial attack \cite{akhtar2018threat,eykholt2018robust, huang2017adversarial, moosavi2016deepfool, moosavi2017universal}.
For example, changing one pixel in an image \blue{may} change the prediction of a CNN \citep{su2019one}. 
Many \blue{researchers} link the vulnerability of CNNs to their Lipschitz properties and the common belief is that CNNs with small Lipschitz constants are more robust against adversarial attack \cite{szegedy2013intriguing, cisse2017parseval, bietti2019kernel, anil2019sorting, virmaux2018lipschitz, fazlyab2019efficient}. Since computing the Lipschitz constants of CNNs is intractable \cite{virmaux2018lipschitz}, existing approaches seek to regularize the norms of individual CNN layers. For example, \citet{cisse2017parseval} proposed Parseval Network where the $\ell_2$ norms of linear and convolutional layers are constrained to be orthogonal. However, from Table~1 in their paper, we can see Parseval Network only slightly improves adversarial robustness in most cases and even reduces robustness in some cases. \citet{anil2019sorting} \blue{combined GroupSort, which is} a gradient norm preserving activation function, with norm-constrained weight matrices regularization to enforce Lipschitzness in fully-connected networks while maintaining the expressive power of the models. %
\citet{li2019preventing} further extended GroupSort to CNNs by proposing Block Convolution Orthogonal Parameterization (BCOP), which restricts the linear transformation matrix of a convolutional kernel to be orthogonal and thus its $\ell_2$ norm is bounded by 1. Again, we find that the improvement of adversarial robustness is typically small while the standard accuracy drops considerably. For example, we use the state-of-the-art adversarial ``Auto Attack'' \cite{croce2020reliable} to test the checkpoint from the authors\footnote{\url{https://github.com/ColinQiyangLi/LConvNet}} and find that, the robust accuracy of their best model on CIFAR-10 is 8.4\% (under standard $\ell_\infty$ attack with $\epsilon=8/255$), which is much smaller than the state of the art (~59.5\%\footnote{\url{https://github.com/fra31/auto-attack}}) such as the methods of \cite{carmon2019unlabeled, wang2019improving, pang2020boosting}, while the standard accuracy drops to 72.2\%. 
Besides, since GroupSort and BCOP have virtually changed the forward computation and/or architecture of the network, it is \emph{unclear} whether their improvement in adversarial robustness is due to regularization of norms or the change in computation/architecture. 
These issues raise concerns over the effectiveness of regularization of norms. %

The approaches of regularization of norms are motivated by the idea that reducing norms of individual layers can reduce global Lipschitz constant and reducing global Lipschitz constant can ensure smaller local Lipschitz constants and thus improve robustness. 
In this paper, we carefully investigate the connections and distinctions between the norms of layers, local Lipschitz constants, and global Lipschitz constants. And our findings, both theoretically and empirically, do not support the prevailing idea that large norms are bad for adversarial robustness. 

Our contribution in this paper is summarized as follows.
\begin{itemize}
	\item We theoretically characterize \norms of 2D multi-channel convolutional layers. \blue{To our knowledge, }our approach is the fastest among the existing methods for computing norms of convolutional layers. 
	\item We present a novel regularization method termed norm decay, which can improve generalization of CNNs.
	\item We \emph{prove} that robust classifiers can be realized with neural networks. \blue{Further,} our theoretical results and extensive experiments suggest that large norms (compared to norm-regularized networks) of CNN layers do not hurt adversarial robustness.
\end{itemize}

\section{Related Work}
Researches related to the norms of convolutional layers are mostly concerned with the $\ell_2$ norm. %
For example, \citet{miyato2018spectral} reshape the 4D convolutional kernel into a 2D matrix and use power iterations to compute the $\ell_2$ norm of the matrix. Although this method can improve the image quality produced by WGAN \cite{arjovsky2017wasserstein}, the norm of the reshaped convolutional kernel does not reflect the true norm of the kernel. 
Based on the observation that the result of power iterations can be computed through gradient back-propagation, \citet{virmaux2018lipschitz} proposed AutoGrad %
to compute the $\ell_2$ norm. 
\citet{sedghi2018the} theoretically analyzed the circulant patterns in the unrolled convolutional kernel, based on which they discovered a new approach to compute the singular values of the kernels. Using the computed spectrum of convolution, they proposed singular value clipping, a regularization method which projects a convolution onto the set of convolutions with bounded $\ell_2$ norms. %
It is worth noting that, because of the equivalence of the matrix norms, i.e., $1/\sqrt m \|A\|_{1}\leq \|A\|_{2}\leq \sqrt n \|A\|_{1}$ and $1/\sqrt n \|A\|_{\infty}\leq \|A\|_{2}\leq \sqrt m \|A\|_{\infty}$ for all matrices $A\in \mathbb{R}^{m\times n}$, our approaches to compute the $\ell_1$ and $\ell_\infty$ norm have the same functionalities as those to compute $\ell_2$ norm, while our approaches are much more efficient. 
\citet{gouk2018regularisation} give an analysis on the $\ell_1$ and $\ell_\infty$ norm of convolutional layers but they neglect the padding and strides of convolution, which may lead to incorrect computation results. %

All these works have not yet given a clear analysis of how the norms of neural net layers are related to adversarial robustness. To bridge this gap, we first characterize the norms of CNN layers and then analyze theoretically and test empirically if large norms are bad for adversarial robustness.

\section{The $\ell_1$ and $\ell_\infty$ Norm of Convolutional Layers}
To understand how norms of CNN layers influence adversarial robustness, we first need to characterize the norms. 
\citet{sedghi2018the} proposed a method for computing the singular values of convolutional layers, where the largest one is the $\ell_{2}$ norm. However, their method applies to only the case when the stride of convolution is 1, and computing singular values with their algorithm is still computationally expensive and prohibit its usage in large scale deep learning. To alleviate these problems, we theoretically analyze the $\ell_{1}$ norm and $\ell_{\infty}$ norm of convolutional layer, and we find that our method of computing norms is much more efficient than that of \cite{sedghi2018the}. %

\begin{figure*}
	\centering
	\begin{tabular}{C{0.3\textwidth}C{0.7\textwidth}}
		\begin{tabular}{cc}
			\begin{tabular}{c}
				\smallskip
				\begin{subfigure}{0.16\textwidth}
					\centering
					\begin{tikzpicture}[scale=0.35, every node/.style={scale=0.74}]
					\draw[fill=white]  (0, 0) rectangle (5,5);
					\draw[fill=white]  (0.5, 0.5) rectangle (5.5,5.5);
					\draw[step=1,black,shift={(0.5,0.5)}] %
					(1,1) grid (6,6);
					\foreach \t in {1,2,...,5}
					{
						\foreach \r in {1,2,...,5}
						{
							\node at (\t+1, 7-\r) {{\floor{(\r-1)*5 + \t}}};
						}
					}
					\node at (1,1) {\rotatebox{45}{$\cdots$}};
					\end{tikzpicture}
					\caption{3D input channels}
					\label{fig:input channel}
				\end{subfigure}
				\\
				\begin{subfigure}{0.16\textwidth}
					\centering
					\begin{tikzpicture}[scale=0.35, every node/.style={scale=0.8}]
					\draw  (0.3,3) -- (0,3) -- (0, 0) -- (3,0) -- (3,0.3);
					\draw  (0.9,3.3) -- (0.3,3.3) -- (0.3, 0.3) -- (3.3,0.3) -- (3.3,0.9);
					\draw[step=1,black,shift={(-0.1,-0.1)}] %
					(1,1) grid (4,4);
					\foreach \t in {1,2,3}
					{
						\foreach \r in {1,2,3}
						{
							\node at (\t+0.4, 4.4-\r) {{\floor{(\r-1)*3 + \t}}};
						}
					}
					\node[scale=0.8] at (0.6,0.6) {\rotatebox{45}{{$\cdots$}}};
					\end{tikzpicture}
					\caption{3D output channels}
					\label{fig:output channel}
				\end{subfigure}
			\end{tabular}
			&
			\begin{subfigure}{0.14\textwidth}
				\centering
				\begin{tikzpicture}[scale=0.43, every node/.style={scale=0.69}]
				\draw  (0.3,3) -- (0,3) -- (0, 0) -- (3,0) -- (3,0.3);
				\draw  (0.9,3.3) -- (0.3,3.3) -- (0.3, 0.3) -- (3.3,0.3) -- (3.3,0.9);
				\draw[step=1,black,shift={(-0.1,-0.1)}] %
				(1,1) grid (4,4);
				\foreach \t in {1,2,3}
				{
					\foreach \r in {1,2,3}
					{
						\draw[fill=green!8] (\t+0.4, 4.4-\r) circle (0.35);
						\node at (\t+0.4, 4.4-\r) {{\floor{(\r-1)*3 + \t}}};
					}
				}
				\node at (0.6,0.6) {\rotatebox{45}{{$\cdots$}}};
				\end{tikzpicture}
				\captionsetup{width=.8\linewidth,font=small,labelfont=small}
				\caption{A 3D ``slice'' (with shape $d_{in} \times 3\times 3$) of a 4D convolutional kernel of shape $d_{out} \times d_{in} \times 3\times 3$}
				\label{fig:conv kernel}
			\end{subfigure}
		\end{tabular}	
		&
		\begin{subfigure}{0.65\textwidth}
			\centering
			\begin{tikzpicture}[scale=0.62, every node/.style={scale=0.75}]
			\draw[step=.5,black] %
			(0,-0.75) grid (13.25,4.5);
			\draw[black,thick] (0,0) rectangle (12.5,4.5);
			\foreach \t in {1,2,...,26}
			{
				\node at (\t/2-0.25, 4.75) {\small{\t}};
			}
			\node at (27/2-0.25, 4.75) {\small{$\cdots$}};
			\foreach \t in {1,2,...,10}
			{
				\node at (-0.25, 4.75-\t/2) {\small{\t}};
			}
			\node at (-0.25, 4.75-11/2) {\small{$\vdots$}};
			
			\foreach \t in {1,2,...,9}
			{
				\foreach \r in {0,1,...,8}
				{
					
					\draw[fill=green!8] (\floor{(\t-1)/3}+\t/2-1/2-0.25+2-\floor{\r/3}-\r/2+4.5, \r/2+0.25) circle (0.2);
					\node at (\floor{(\t-1)/3}+\t/2-1/2-0.25+2-\floor{\r/3}-\r/2+4.5, \r/2+0.25) {\small{\t}};
				}
			}
			
			\node at (0.5, 5.65) {\textbf{Input}};
			\node at (-1, 3.75) {\rotatebox{90}{\textbf{Output}}};
			
			\draw[decoration={calligraphic brace,amplitude=5pt,aspect=0.35}, decorate, line width=1pt]
			(-0.45, 0.1) -- (-0.45, 4.4);
			\node at (-1.1, 1.6) {\rotatebox{90}{Channel $1$}};
			\node at (-0.95, -0.5) {$\vdots$};
			
			\draw[decoration={calligraphic brace,amplitude=5pt}, decorate, line width=1pt]
			(0.1, 5) -- (12.4, 5);
			\node at (6.25, 5.7) {{Channel $1$}};
			\node at (12.6, 5.7) {{Channel $2 \, \cdots$}};
			\end{tikzpicture}
			\captionsetup{width=.85\linewidth,font=small,labelfont=small}
			\caption{The upper left part of the linear transformation matrix of the 2D multi-channel convolutional layer for (a) (b) (c) with stride 1 and no padding. {With padding and various strides, the pattern of the linear transformation matrix is more complicated, but these have been properly addressed in our theorem}.}
			\label{fig:conv matrix}
		\end{subfigure}	
	\end{tabular}
	\caption{An illustration of the linear transformation matrix of a convolutional layer.}
	\label{fig:conv}
\end{figure*}
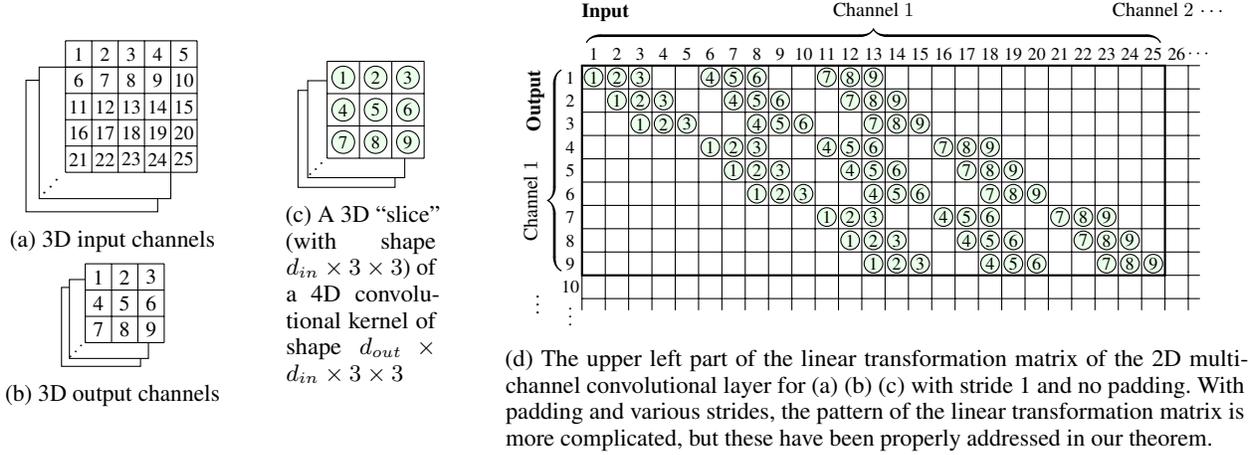

Since 2D multi-channel convolutional layers (Conv2d) \citep{goodfellow2016deep} are arguably the most widely used convolutional layers in practice, we analyze Conv2d in this paper while the analysis for other types of convolutional layer should be similar. 
\begin{setting*}
	Let $\conv \colon \mathbb{R}^{d_{in} \times h_{in} \times w_{in}} \to \mathbb{R}^{d_{out} \times h_{out} \times w_{out}}$ be a 2D multi-channel convolutional layer with a 4D kernel $K \in \mathbb{R}^{d_{out} \times d_{in} \times k_1 \times k_2}$, where $d$ is the channel dimension, $h$ and $w$ are the spatial dimensions of images, and $k_1$ and $k_2$ are the kernel size. Suppose the vertical stride of $\conv$ is $s_1$ and horizontal stride is $s_2$, and padding size is $p_1$ and $p_2$. 
\end{setting*}
We first note that Conv2d without bias is a linear transformation, which can be verified by checking $\conv(\alpha x) = \alpha \conv(x)$ and $\conv(x + y) = \conv(x) + \conv(y)$ for any $\alpha \in \mathbb{R}$ and any tensors $x$ and $y$ with appropriate shape. 
Normally, the input and output of Conv2d are 3D tensors (e.g., images) while the associated linear transformation takes 1D vectors as input. So we reshape the input into a vector (only reshaping the input channel \emph{excluding padding} since padding elements are not variables) and then Conv2d can be represented by $\conv(x) = M x + b$, where $M$ is the linear transformation matrix and $b$ is the bias vector. Then the norm of Conv2d is just the norm of $M$. 
We first state the following well known facts about the norms of a matrix $A \in \mathbb{R}^{m \times n}$:
$\|A\|_1 =\max _{1\leq j\leq n}\sum _{i=1}^{m}|A_{ij}|$, $\|A\|_{\infty}=\max _{1\leq i\leq m}\sum _{j=1}^{n}|A_{ij}|$, and $\|A\|_{2}=\sigma _{\max }(A)$, 
where $\sigma_{\max }(A)$ is the largest singular value of $A$. 
While the exact computation of $M$ is complicated, we can analyze how the norm $\| M \|_{p}$ is related to the convolutional kernel $K$, which is a 4D tensor in the case of Conv2d. 

By carefully inspecting how the output elements of Conv2d are related to the input elements, we find $M$ is basically like the matrix in Figure~\ref{fig:conv matrix}. The rows of $M$ can be formed by convolving a 3D ``slice'' (see Figure~\ref{fig:conv kernel}) of the 4D kernel with the 3D input channels and inspecting which elements on the input channels are being convolved with the 3D kernel slice. If the stride of convolution is 1, $M$ is indeed a doubly circulant matrix like the one in Figure~\ref{fig:conv matrix} \cite{goodfellow2016deep, sedghi2018the}. However, when the stride is not 1 or there is padding in the input channel, the patterns in $M$ could be much more complicated, which is not addressed in existing \emph{analytical} formulas \cite{gouk2018regularisation, sedghi2018the}. We take stride and padding into account and properly address these issues. 
To obtain a theoretical result of the Lipschitz properties of Conv2d, we present the following assumption, which basically means that the convolutional kernel can be completely covered by the input channel (excluding padding) during convolution. 
We emphasize that the assumption holds for most convolutional layers used in practice. 

\begin{assumption} \label{assum}
	Let $c_1$ and $c_2$ be the smallest positive integers such that $c_1 s_1 \geq p_1$ and $c_2 s_2 \geq p_2$. Assume $k_1 + c_1 s_1 - p_1 \leq h_{in}$ and $k_2 + c_2 s_2 - p_2 \leq w_{in}$, and the padding (if any) for the input of $\conv$ is zero padding.
\end{assumption}
We need the following lemma to present our formula to compute the $\ell_1$ norm of Conv2d. 
The overall idea of the lemma is that it links the nonzero elements of every column of $M$ to the elements in the convolutional kernel, which 
is very useful because the $\ell_{1}$ norm of $M$ is exactly the maximum of the absolute column sum of $M$. 
\begin{lemma}\label{lem:kernel}
	Suppose Assumption~1 holds. The indices set for the last two dimensions of $K$ is $\mathcal{N} \ceq \{(k,t) \colon 1 \leq k \leq k_1, 1 \leq t \leq k_2\}$.
	Let $\sim$ be a binary relation on $\mathcal{N}$ such that, if indices $(a, b)$ and $(c, d)$ satisfy $(a - c) \equiv 0 \Mod{s_1}$ and $(b - d) \equiv 0 \Mod{s_2}$, then $(a, b) \sim (c, d)$. 
	Let $\mathcal{A}_{(a,b)} \subseteq \mathcal{N}$ denote the largest set\footnote{By largest set we mean adding any other indices to $\mathcal{A}_{(a,b)}$ would violate the conditions that follow.} of indices such that $(a, b) \in \mathcal{A}_{(a,b)}$ and for all $(c, d) \in \mathcal{A}_{(a,b)}$, $(c, d) \sim (a,b)$ and $ 0 \leq c - a \leq h_{in} + 2p_1 - k_1$ and $ 0 \leq d - b \leq w_{in} + 2p_2 - k_2$. Let $\mathcal{S}$ be a set of indices sets defined as $\mathcal{S} \ceq \{\mathcal{A}_{(a,b)} \colon (a,b) \in \mathcal{N}\}$. 
	Let $M_{:,n}$ be the $n$-th column of the linear transformation matrix $M$ of $\conv$, and let $\nz(M_{:,n})$ be the set of nonzero elements of $M_{:,n}$. 
	Then for $n = 1, 2, \dots, d_{in} h_{in} w_{in}$, there exists an indices set $\mathcal{A} \in \mathcal{S}$ such that $\nz(M_{:,n}) \subseteq \{K_{i,j, k, t} \colon 1 \leq i \leq d_{out}, (k, t) \in \mathcal{A}\}$, where $j = \lceil n / (h_{in} w_{in}) \rceil$.
	Furthermore, for $j = 1,2, \dots, d_{in}$, for all $\mathcal{A} \in \mathcal{S}$, there exists a column $M_{:,n}$ of $M$, where $(j-1)h_{in}w_{in} < n \leq j h_{in}w_{in} $, such that $\nz(M_{:,n}) \supseteq \{K_{i,j, k, t} \colon 1 \leq i \leq d_{out}, (k, t) \in \mathcal{A}\}$.
\end{lemma} 
Now we are ready to show how to calculate the norms of Conv2d. 

\begin{theorem} \label{thm:norm}
	Suppose Assumption~\ref{assum} holds. Then \norms and an upper bound of the $\ell_{2}$ norm of $\conv$ are given by
	\begin{align}
	&\| \conv \|_{1} = \max_{1 \leq j \leq d_{in}} \max_{\mathcal{A} \in \mathcal{S}} \sum_{(k,t) \in \mathcal{A}} \sum_{i=1}^{d_{out}} | K_{i,j,k,t} | , \label{eq:1norm} \\
	&\| \conv \|_{\infty} = \max_{1 \leq i \leq d_{out}} \sum_{j=1}^{d_{in}} \sum_{k=1}^{k_1} \sum_{t=1}^{k_2} | K_{i,j,k,t} | , \label{eq:inf-norm} \\
	&\| \conv \|_{2} \leq \bigg(h_{out} w_{out} \sum_{i=1}^{d_{out}} \sum_{j=1}^{d_{in}} \sum_{k=1}^{k_1} \sum_{t=1}^{k_2} | K_{i,j,k,t} |^2\bigg)^{\frac{1}{2}} \label{eq:2norm}
	\end{align}
	where $\mathcal{S}$ is a set of indices sets defined in Lemma~\ref{lem:kernel}.
\end{theorem}

The proofs of Lemma~\ref{lem:kernel} and Theorem~\ref{thm:norm} are lengthy and deferred to the Appendix.

\section{Do Large Norms Hurt Adversarial Robustness?}
Many works mentioned in the Introduction regularize the norms of layers to improve robustness, while some authors \cite{sokolic2017robust, weng2018evaluating, yang2020closer} pointed out that local Lipschitzness is what really matters to adversarial robustness. In the setting of neural networks, the relations and distinctions between global Lipschitzness, local Lipschitzness, and the norms of layers are unclear. We devote this section to investigate their connections. For completeness, we provide the definition of Lipschitz constant.
\begin{definition}[Global and local Lipschitz constant] \label{def:lip const}
	Given a function $f \colon \mathcal{X} \to \mathcal{Y}$, where $\mathcal{X}$ and $\mathcal{Y}$ are two  finite-dimensional normed spaces equipped with norm $\|\cdot\|_{p}$, the global Lipschitz constant of $f$ is defined as
	\begin{equation}
		\|f\|_{p} \ceq \sup_{x_1, x_2 \in \mathcal{X}} \frac{\|f(x_{1}) - f(x_{2})\|_{p}}{\|x_{1} - x_{2}\|_{p}}.
	\end{equation}
	We call $\|f\|_{p}$ a local Lipschitz constant on a compact space $\mathcal{V} \subset \mathcal{X}$ if $x_1$ and $x_2$ are confined to $\mathcal{V}$. In the context of neural nets, the norm is usually the $\ell_{1}$, $\ell_{2}$, or $\ell_{\infty}$ norm.
\end{definition}

To deduce the prevailing claim that large norms hurt adversarial robustness, one must go through the following reasoning: large norms of layers $\to$ large global Lipschitz constant of the network $\to$ large local Lipschitz constant in the neighborhood of samples $\to$ the output of the network changes so sharply around samples that the prediction is changed $\to$ reducing adversarial robustness. 
However, there are at least two serious issues at the first and second arrow in the above reasoning. The first issue is that large norms of individual layers do not necessarily cause the global Lipschitz constant of the network to be large, as demonstrated in the following proposition. 
\begin{proposition}\label{prop:coupling}
	There exists a feedforward network with ReLU activation where the norms of all layers can be arbitrarily large while the Lipschitz constant of the network is 0. 
\end{proposition}

The proof is deferred to the Appendix. 
Although the network illustrated in the proof of Proposition~\ref{prop:coupling} is a very simple one, it does show that the coupling between layers could make the actual Lipschitz constant of a neural net much smaller than we can expect from the norms of layers. A related discussion of coupling between layers is presented in \cite{virmaux2018lipschitz}. This proposition breaks the logical chain at the first arrow in the above reasoning of large norms hurting adversarial robustness. 
The second issue in the reasoning is that, even if the Lipschitz constant of a neural network is very large, it can still be adversarially robust. This is because, \emph{local} Lipschitzness, which means the output of a network does not change sharply in the neighborhood of samples, is \emph{already sufficient} for adversarial robustness, and it has no requirement on the \emph{global} Lipschitz constant \cite{sokolic2017robust, weng2018evaluating, yang2020closer}. 
In the next paragraph, we will first prove that under a mild assumption, robust classifiers can be achieved with neural networks, and then we will prove that the Lipschitz constant of a \emph{robust} classifier can be arbitrarily large. 

Since we are primarily interested in classification tasks, our discussion will be confined to these tasks. We first need some notations. 
Let $\mc{X} \subset \mathbb{R}^n$ be the instance space (data domain) and $\mc{Y}=\{1,\dots, C\}$ be the (finite) label set where $C$ is the number of classes. Let $\mc{D}$ be the probability measure of $\mc{X}$, i.e., for a subset $A \subset \mc{X}$, $\mc{D}(A)$ gives the probability of observing a data point $x \in A$. Let $\mc{X}$ be endowed with a metric $d$ that will be used in adversarial attack, and let $\ball(x, \epsilon) \cq \{\tilde{x} \colon d(x, \tilde{x}) \leq \epsilon\}$ be the $\epsilon$-neighborhood of $x$. 
Let $f\colon \mc{X} \to \mc{Y}$ denote the underlying labeling function (which we do not know), and let $\mv{X}{c} \subset \mc X$ be the set of class~$c$. The robust accuracy is defined as follows, similar to the ``astuteness'' in \cite{wang2018analyzing, yang2020closer}. 
\begin{definition}[Robust accuracy]
	We say a classifier $g \colon \mathbb{R}^n \to \mathbb{R}$ have robust accuracy $\gamma$ under adversarial attack of magnitude $\epsilon \geq 0$ if $\gamma = \mc{D}\big(\{x \in \mc{X} \colon |g(\tilde{x}) - f(x)| < 0.5 \ttt{ for all } \tilde{x} \in \ball(x, \epsilon)\}\big)$.
\end{definition}
Here, for convenience of proof, we use a classifier that outputs a real number, and its prediction is determined by choosing the nearest label to its output. Thus, if the output of $g$ is at most 0.5 apart from the true label, then $g$ gives the correct label. This definition and the following theorem and proposition can be easily generalized to the widely used classifiers with vectors as outputs. 
{Intuitively, robust accuracy is the probability measure of the set of ``robust points'', which are the points whose $\epsilon$-neighbors can be correctly classified by $g$.} 
Our next theorem shows that, under a mild assumption similar to that in \cite{yang2020closer}, there exits a neural network that can achieve robust accuracy 1 (i.e., the highest accuracy). 
\begin{assumption}[2-epsilon separable]\label{assum:separate}
	The data points of any two different classes are 2-epsilon separable: $\inf\{d(x^{(i)}, x^{(j)}) \colon x^{(i)} \in \mv X i, x^{(j)} \in \mv X j, i\neq j\} > 2\epsilon$.
\end{assumption}
Intuitively, Assumption~\ref{assum:separate} states any two epsilon-balls centered at data points from two different classes do not have overlap. 
We would like to provide an explanation for why the assumption holds for a reasonable attack size $\epsilon$ in computer vision tasks. 
We say the attack size $\epsilon$ is reasonable, if for all $x \in \mc{X}$ and for all $s \in \ball(x, \epsilon)$, the label of $s$ given by humans is the same as that of $x$. 
Thus, if $\epsilon$ is reasonable (as in our definition), the two balls $\ball(x_1, \epsilon)$ and $\ball(x_2, \epsilon)$ for $x_1$ and $x_2$ coming from two different classes would not have overlap, which means the 2-epsilon separable assumption should hold for a \emph{reasonable} $\epsilon$. In our analysis, we do not rely on the number of classes, so the assumption should hold for any number of classes. But we do think in reality, the training of adversarially robust classifiers may be more difficult for larger number of classes because intuitively, the neighborhood $\ball(x, \epsilon)$ of $x$ from different classes are more likely to be close to each other if the number of classes are larger. 

\begin{theorem}[Realizability of robust classifiers]\label{thm:realizability}
	Let $\rho \colon \mathbb{R}\to\mathbb{R}$ be any non-affine continuous function which is continuously differentiable at at least one point, with nonzero derivative at that point. If Assumption~\ref{assum:separate} holds, then there exists a feedforward neural network with $\rho$ being the activation function that has robust accuracy $1$.
\end{theorem}

The proof is deferred to the Appendix. 
We notice that \citet{yang2020closer} showed a related result that there exists a function that has small local Lipschitz constants and achieves robust accuracy 1. Our result (Theorem~\ref{thm:realizability}) is different from theirs in that we prove that a neural network that can be realized in a digital computer can obtain robust accuracy 1 while they proved an \emph{abstract} function $f$ can obtain robust accuracy 1, where the definition of $f$ relies on knowing the data distribution $\mc{D}$ and $f$ may not be realized in a digital computer. 
\citet{yang2020closer} also empirically showed that real-world image datasets are typically $2\epsilon$-separable and thus there should exist neural networks that achieve high robust accuracy. 
Using Theorem~\ref{thm:realizability}, we are ready to show that a neural network having robust accuracy 1 can have arbitrarily large Lipschitz constant, as in the following proposition.
\begin{proposition}\label{thm:large}
	Let $\rho \colon \mathbb{R}\to\mathbb{R}$ be any non-affine continuous function which is continuously differentiable at at least one point, with nonzero derivative at that point. 
	If Assumption~\ref{assum:separate} holds, then for all $\xi > 0$, there exists a feedforward neural network with $\rho$ being the activation function that achieves robust accuracy $1$ and its Lipschitz constant is at least $\xi$.
\end{proposition}

The proof is deferred to the Appendix. 
Proposition~\ref{thm:large} shows that neural networks that have large Lipschitz constant can be adversarially robust because they can have small local Lipschitz constants \emph{in the instance domain}. This proposition implies that what really matters is the local Lipschitz property of the network instead of the global one. 
\citet{yang2020closer} also stressed the importance of controlling local Lipschitzness of neural nets, by showing a function that has small local Lipschitz constant can achieve robust accuracy 1. 

On the other hand, although enforcing a small global Lipschitz constant can ensure local Lipschitzness, it may reduce the expressive power of the network and hurt standard accuracy. Let us consider fitting the function $f(x) = 1 / x$ in the interval $(0.5, 1)$; then no 1-Lipschitz function could fit it well since the slope of the function in that interval is as large as 4. Thus, enforcing global Lipschitzness may result in hurting standard accuracy a lot while obtaining only a slight improvement in robustness (e.g., as in \cite{li2019preventing}). 
In order to further investigate how norms influence the adversarial robustness in practice, we further propose a novel norm-regularization method in the next section. 

\section{A Regularization Method: Norm Decay}
Equipped with Eq.~\eqref{eq:1norm} and Eq.~\eqref{eq:inf-norm}, we present an algorithm termed norm decay to control (or regularize) the norm of fully-connected layers and convolutional layers. Then we investigate how norm decay influences generalization and adversarial robustness in experiments. 

The norm decay approach is to add a regularization term to the original loss function $\loss(\theta)$, where $\theta$ is the parameter, to form an augmented loss function:
\begin{equation}
	\min_{\theta} \quad \loss(\theta) + \frac{\beta}{N} \sum_{i=1}^{N} \|\theta^{(i)}\|_{p}
\end{equation}
where $\theta^{(i)}$ denotes the linear transformation matrix of the $i$-th layer and $\beta$ is a hyperparameter, and the summation is over all fully-connected layers and convolutional layers. %

Form Eq.~\eqref{eq:1norm} and Eq.~\eqref{eq:inf-norm}, we can see that the $\ell_1$ and $\ell_\infty$ norm depends on only some elements in the kernel, which means the gradient of norm w.r.t. kernel elements ($\nabla_{\theta} \|\theta^{(i)}\|_{p}$) are typically sparse. 
Besides, since the norm is the sum of the absolute values of these elements, the gradient w.r.t. a single kernel element is either 1 or -1 or 0, which makes the computation of gradient very efficient. 
After updating the kernel parameters using an optimizer such as stochastic gradient descent (SGD), the elements that contribute to the norm may become completely different from those before the update (due to the $\max$ operation in Eq.~\eqref{eq:1norm} and Eq.~\eqref{eq:inf-norm}), which could cause non-smoothness (i.e., rapid change) of the gradient $\nabla_{\theta} \|\theta^{(i)}\|_{p}$. To smooth the gradient change and stabilize training, we introduce a momentum $\gamma$ to keep a moving average of the gradient of the norms. The details are shown in Algorithm~\ref{alg:norm-decay}. 

\begin{algorithm}
	\begin{algorithmic}[1]
		\renewcommand{\algorithmicrequire}{\textbf{Input:}}
		\renewcommand{\algorithmicensure}{\textbf{Output:}}
		\Require loss function $\loss$ (assuming it is to be minimized), parameters $\theta$, momentum $\gamma$, regularization parameter $\beta$
		\Ensure parameters $\theta$
		\State $h \gets \mathbf{0}$ (initialize the gradient of norms of layers)
		\Repeat
		\State $g \gets \nabla_{\theta}\loss $
		\State Compute $p$, the gradient of $\ell_1$ or $\ell_{\infty}$ norm of each fully-connected and convolutional layer
		\State $h \gets \gamma \cdot h + (1-\gamma) \cdot p$
		\State $g \gets g + \beta/N \cdot h$
		\State $\theta \gets \sgd(\theta, g)$
		\Until {convergence}
	\end{algorithmic}
	\caption{Norm Decay}
	\label{alg:norm-decay}
\end{algorithm}

\section{Experiments}
Firstly, we show our approaches for computing norms of Conv2d are very efficient. 
In the second part, we conduct extensive experiments to investigate if regularizing the norms of CNN layers is effective in improving adversarial robustness. In the third part, we compare the norms of the layers of adversarially robust CNNs against their non-adversarially robust counterparts. 

\subsection{Algorithmic Efficiency Comparison}
We compare the efficiency of three methods that can compute the exact norms of convolutional layers, including computing the $\ell_2$ norm with power iteration \ccite{virmaux2018lipschitz} and circulant matrix \ccite{sedghi2018the} and computing \norms with Eq.~\eqref{eq:1norm} and Eq.~\eqref{eq:inf-norm}. The result is shown in Table~\ref{tab:timing}, which shows that our approaches are much faster (up to 14,000 times faster) than the others, while our approaches are theoretically and empirically equivalent to the others in computing norms.

\begin{table}[tbp]
	\centering
	\resizebox{\linewidth}{!}{%
		\begin{tabular}{lcccc}
			\toprule
			kernel size & $\ell_2$(VS) & $\ell_2$(SGL) & $\ell_1$(ours) & $\ell_\infty$(ours) \\
			\midrule
			3, 3, 32, 32 & 26.5  & 5.75  & 0.00605 & 0.00576 \\
			3, 3, 32, 128 & 27.4  & 6.92  & 0.00682 & 0.00575 \\
			3, 3, 128, 256 & 29.0  & 98.0  & 0.00576 & 0.00560 \\
			3, 3, 256, 512 & 59.4  & 490 & 0.0117 & 0.00898 \\
			5, 5, 256, 128 & 59.7  & 91.5  & 0.0103 & 0.00729 \\
			5, 5, 512, 256 & 255 & 523 & 0.0239 & 0.0180 \\
			\bottomrule
	\end{tabular}}
	\caption{Computation time (seconds) of 100 runs of computing different norms for various kernels. The experimental setup is shown in the next subsection and the computation is run on GPU. The input image has the same shape as a CIFAR-10 image. The kernel size is represented by (kernel height, kernel width, \# input channels, \# output channels). VS denotes the method of \ccitet{virmaux2018lipschitz} and SGL denotes the method of \ccitet{sedghi2018the}.}
	\label{tab:timing}
\end{table}

\begin{table*}[!thb]
	\centering
	\resizebox{\textwidth}{!}{%
		\begin{tabular}{c|c|c|cccc|cccc|cccc|cccc}
			\toprule
			\multicolumn{1}{c}{} & & plain & \multicolumn{4}{c|}{weight decay} & \multicolumn{4}{c|}{singular value clipping} & \multicolumn{4}{c|}{$\ell_1$ norm decay} & \multicolumn{4}{c}{$\ell_\infty$ norm decay} \\
			\midrule
			model & ACC & --- & $10^{-2}$ & $10^{-3}$ & $10^{-4}$ & $10^{-5}$ & 0.5 & 1.0 & 1.5 & 2.0 & $10^{-2}$ & $10^{-3}$ & $10^{-4}$ & $10^{-5}$ & $10^{-2}$ & $10^{-3}$ & $10^{-4}$ & $10^{-5}$ \\
			\midrule
			\multirow{2}[0]{*}{vgg}&
			Clean & 90.4 & 91.6 & 91.7 & 90.1 & 90.2 & 87.6 & 89.1 & 90.0 & 89.9 & 88.1 & 91.1 & 90.6 & 90.8 & 91.8 & 91.1 & 90.8 & 90.6
			\\
			& Robust & 60.2 & 56.3 & 60.5 & 60.6 & 60.3 & 48.8 & 52.2 & 54.1 & 56.7 & 56.5 & 62.5 & 61.1 & 60.1 & 56.9 & 60.0 & 60.8 & 60.1
			\\
			\midrule
			\multirow{2}[0]{*}{resnet}&
			Clean & 93.2 & 94.3 & 94.1 & 93.1 & 92.7 & 93.6 & 94.0 & 94.2 & 93.8 & 92.5 & 93.4 & 93.5 & 93.4 & 93.0 & 93.8 & 93.1 & 93.0
			\\
			& Robust & 37.0 & 28.2 & 33.7 & 33.9 & 40.9 & 35.2 & 41.7 & 43.2 & 39.8 & 24.5 & 37.7 & 38.3 & 37.5 & 20.0 & 34.7 & 38.9 & 37.6
			\\
			\midrule
			\multirow{2}[0]{*}{senet}&
			Clean & 93.1 & 94.2 & 93.9 & 93.0 & 92.4 & 93.8 & 94.2 & 93.8 & 94.2 & 92.3 & 93.8 & 93.3 & 93.3 & 93.0 & 93.6 & 92.8 & 93.2
			\\
			& Robust & 35.7 & 23.5 & 32.8 & 37.0 & 34.8 & 30.5 & 35.6 & 35.2 & 37.4 & 33.6 & 36.0 & 38.2 & 36.7 & 28.6 & 31.0 & 37.6 & 37.4
			\\
			\midrule
			\multirow{2}[0]{*}{regnet}&
			Clean & 91.8 & 93.6 & 94.4 & 92.3 & 91.3 & 93.9 & 93.4 & 93.0 & 92.4 & 93.7 & 92.3 & 91.6 & 91.9 & 93.4 & 92.0 & 91.8 & 91.9
			\\
			& Robust & 34.8 & 23.7 & 30.3 & 30.0 & 31.0 & 27.7 & 28.8 & 29.0 & 28.8 & 29.2 & 31.1 & 28.1 & 34.3 & 23.2 & 27.7 & 27.9 & 30.6
			\\
			\bottomrule
	\end{tabular}}
	\caption{Comparison of clean accuracy (\%) and robust accuracy (\%) of 4 CNN models trained with different norm-regularization methods on CIFAR-10. The second row corresponds to the values of regularization parameters. Robust accuracy is tested with standard Auto Attack \cite{croce2020reliable} under $\ell_\infty$ metric with $\epsilon=1/255$.}
	\label{tab:acc}
\end{table*}

\subsection{Regularizing Norms Improves Generalization but Can Hurt Adversarial Robustness} \label{sec:reg}
To better understand the effect of regularizing the norm of CNN layers, we conduct experiments with various models on \cifar\ \cite{krizhevsky2009learning}. %
Specially, we use three approaches, including weight decay (WD), singular value clipping (SVC) \cite{sedghi2018the}, and norm decay (ND), to regularize the norms. Here, we only use the norm-regularization methods that do not change the architecture of the network, and thus exclude the GroupSort \cite{anil2019sorting} and BCOP \cite{li2019preventing}. We also exclude the methods that may not regularize the true norms (e.g., reshaping the convolutional kernel into a matrix) such as Parseval Regularization \cite{cisse2017parseval} and \cite{gouk2018regularisation}. 

\noindent\textbf{Experimental setup.} We set the regularization parameter to different values and test generalization and adversarial robustness of the models on test set. In norm decay, we simply set the hyperparameter $\gamma$ (momentum) to 0.5 and test the other hyperparameter $\beta$ in $\{10^{-5}, \dots, 10^{-2}\}$. We also test the regularization parameter of weight decay in $\{10^{-5}, \dots, 10^{-2}\}$ and test SVC by clipping the singular values to $\{2.0, 1.5, 1.0, 0.5\}$, respectively, following the setting in the original paper. 
We use four CNN architectures in our experiments, including VGG-11 \cite{simonyan2014very}, ResNet-18 \cite{he2016deep}, SENet-18 \cite{hu2018squeeze}, and RegNetX-200MF \cite{radosavovic2020designing}.
We use the SGD optimizer with momentum of 0.9 and set the initial learning to 0.01. We train the models for 120 epochs and decay the learning rate by a factor of 0.1 at epoch 75, 90, and 100. After finishing training, we use the state-of-the-art attack ``Auto Attack'' \cite{croce2020reliable} to attack the trained CNNs. The experiments are conducted on a machine a GTX 1080 Ti GPU and an Intel Core i5-9400F 6-core CPU and 32GB RAM. 

The result is shown in Table~\ref{tab:acc}. Since we find that all models trained with WD, SVC, and ND have basically zero robust accuracy under $\ell_\infty$ attack with $\epsilon=8/255$ and $\epsilon=4/255$, we set $\epsilon=1/255$ to see the actual effect of regularizing norms. Because of that, we first conclude that these regularization methods cannot improve adversarial robustness by reducing norms when facing large attack (in the sense of large $\epsilon$). 
From Table~\ref{tab:acc}, we can see that the four regularization methods typically improve generalization. However, as the regularization becomes stronger, the norm of all layers becomes smaller (see Appendix for the changes of norms during training) while the robust accuracy could slightly decrease. 
The reduction in robust accuracy is especially evident when the regularization is the strongest and the norms are the smallest (in the first column of each regularization method in Table~\ref{tab:acc}). 
This result is very surprising and contradicts the prevailing claim that small norms of CNN layers improve robustness \cite{szegedy2013intriguing, cisse2017parseval, anil2019sorting, li2019preventing}. 
We can see that there seems to be a trade-off between standard (clean) accuracy and robust accuracy. When the clean accuracy gets a higher value, the robust accuracy typically gets a lower value. This trade-off has been pointed out by \citet{tsipras2018robustness}, and they proved that the trade-off is inevitable when the distribution of two different classes is ``mixed''. However, \citet{yang2020closer} have shown that the CIFAR-10 training set and test set are both $2\epsilon$-separable for $\epsilon$ much larger than the typical values used in adversarial attack. Therefore, by Theorem~\ref{thm:realizability}, there should exist a neural network that achieves robust accuracy 1 and there should be no \emph{intrinsic} trade-off.  

The reason for this phenomenon may be that regularizing the norms in fact suppresses the power of CNNs to become local Lipschitz. From the results in the last section, we know that large norms do not necessarily result in large local Lipschitz constants. Thus, in an unconstrained parameter space (in the case of no regularization) the network may be able to find a minimizer (w.r.t. the loss) that has better local Lipschitzness. When the parameter space is constrained (due to regularization), the network may need to sacrifice local Lipschitzness to retain standard accuracy, which is the training target. 

Although the proposed norm decay \emph{may slightly} reduce the adversarial robustness, it still serves as a novel and promising regularizer for CNNs in improving standard generalization. %

\subsection{The Norms of Adversarially Robust Networks}
Equipped with our efficient approaches to computing norms of convolutional layers, we further test how the norms of adversarially robust CNNs differ from their non-adversarially robust counterparts. 
Specifically, we use three adversarial training frameworks, namely, PGD-AT \cite{madry2017towards}, ALP \cite{kannan2018adversarial}, and TRADES \cite{zhang2019theoretically} to train the four models, namely, VGG-11, ResNet-18, SENet-18, and RegNetX-200MF. %
The experimental setting is the same as that in the last subsection except the initial learning rate is set to 0.1 by following the setting of \citet{pang2020boosting}. After finishing training, we compute the $\ell_\infty$ norms of all layers in the CNNs with/without adversarial training. The result is shown in Figure~\ref{fig:norm}. We can see that the norms of layers of adversarially robust CNNs are comparable or even larger than their non-adversarially robust counterparts (e.g., the adversarially robust ResNet and SENet have especially larger norms while having much higher robust accuracy than the plain models). Due to space limitation, we put the comparison of the norms of individual layers in the supplementary material.  
These findings consistently show that large norms of CNNs do not hurt adversarial robustness and what really matters is the local Lipschitzness of the networks. 

\begin{figure*}[!thb]
	\begin{center}
		\includegraphics[width=0.5\linewidth]{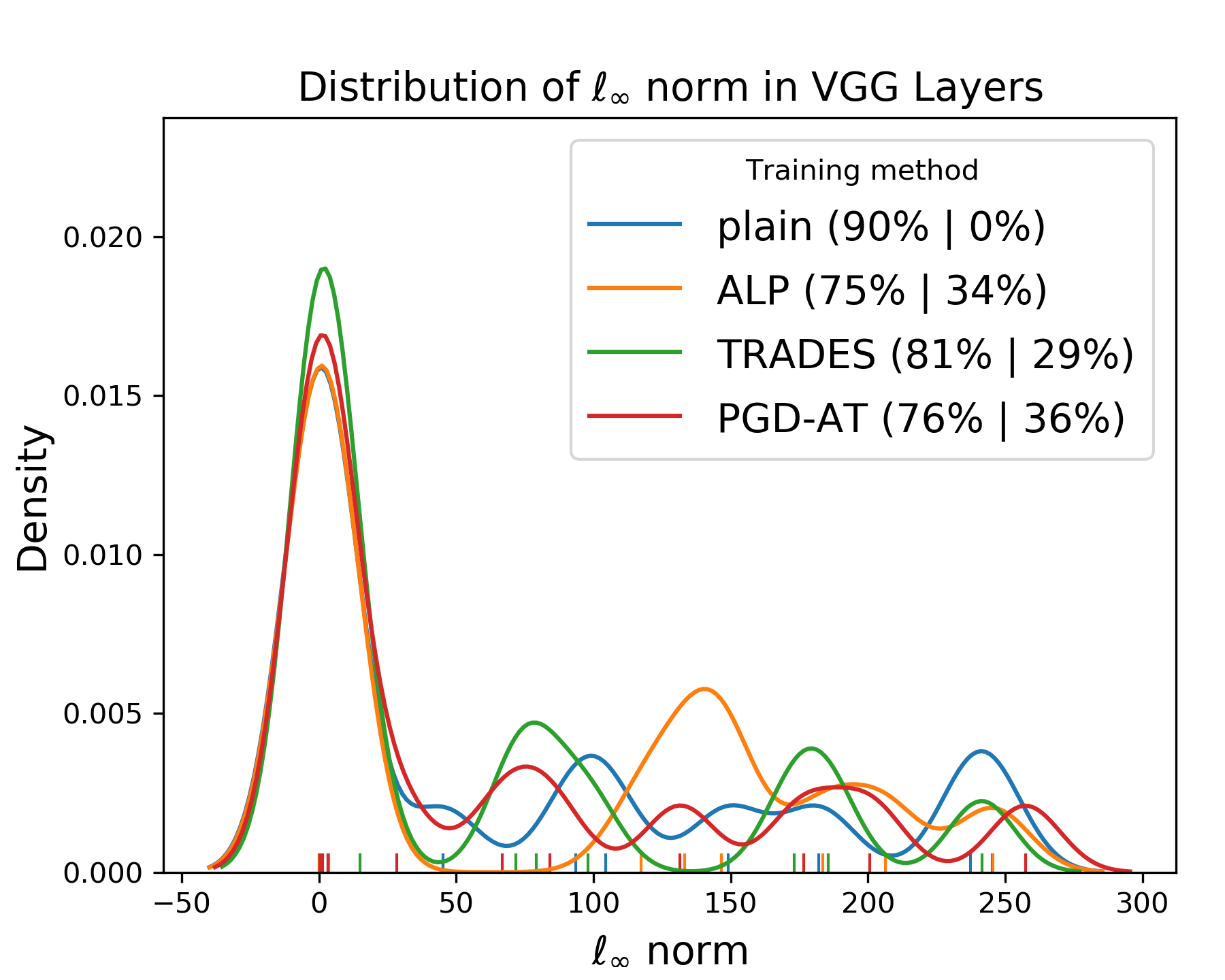}\hfil
		\includegraphics[width=0.5\linewidth]{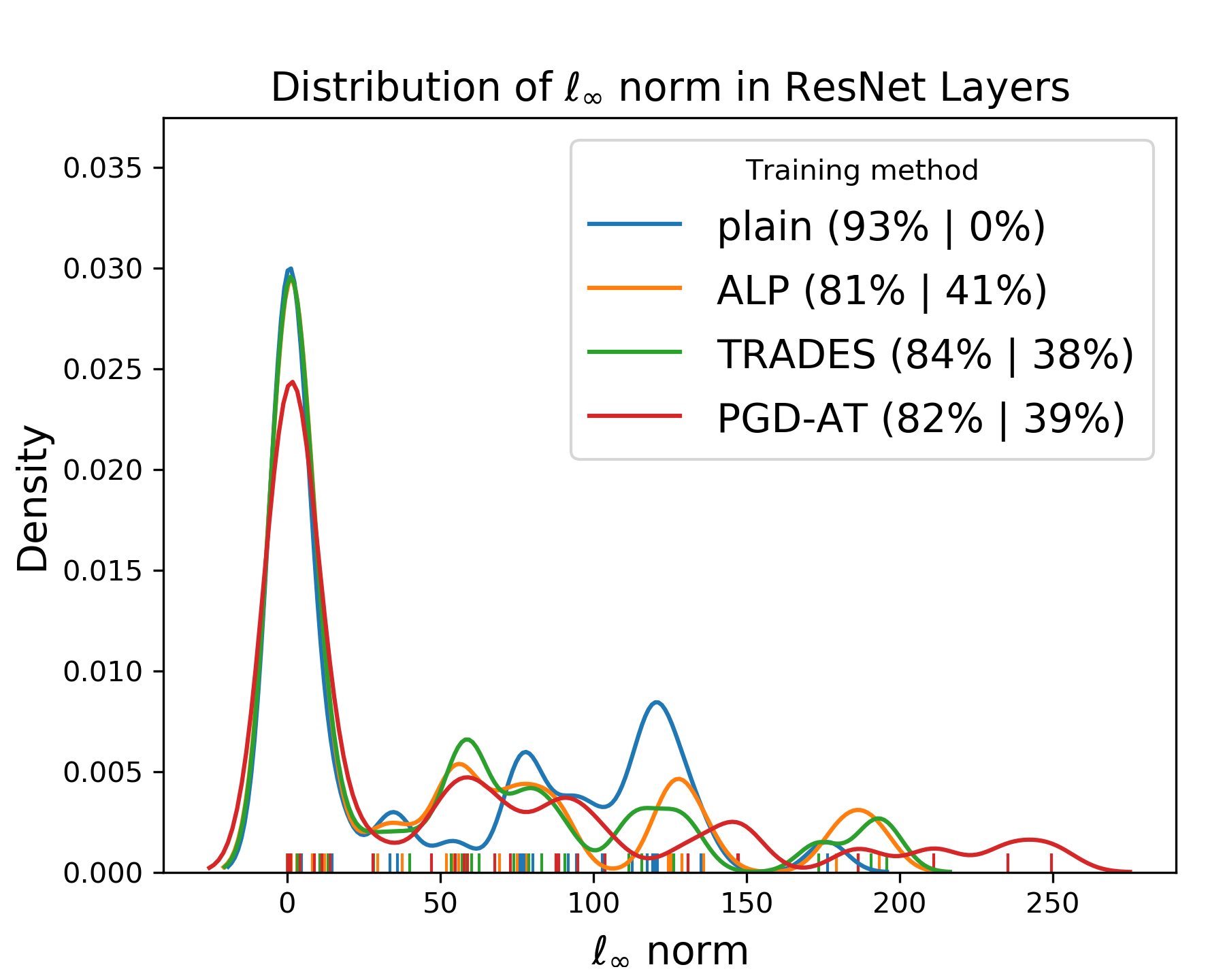}\\
		\includegraphics[width=0.5\linewidth]{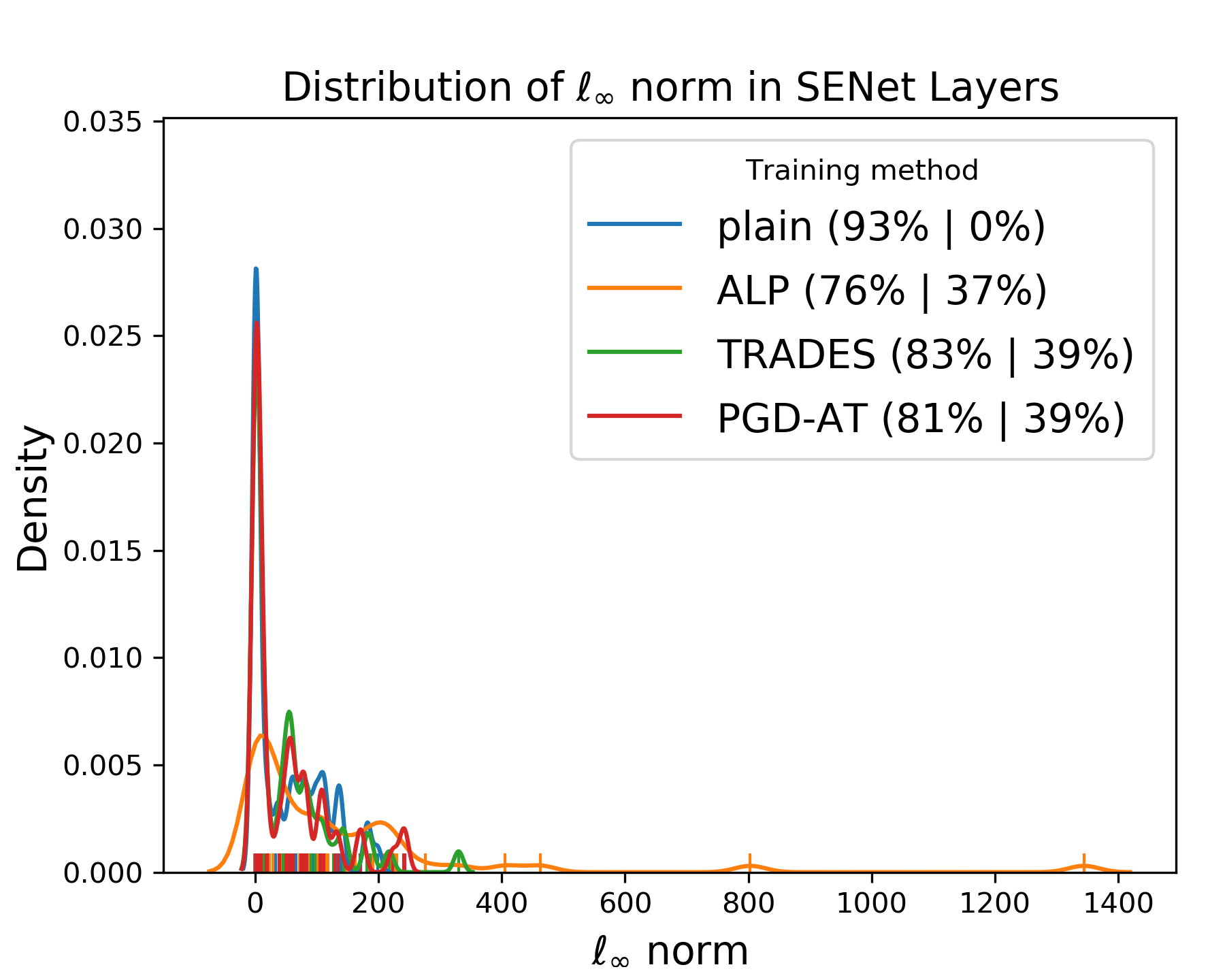}\hfil
		\includegraphics[width=0.5\linewidth]{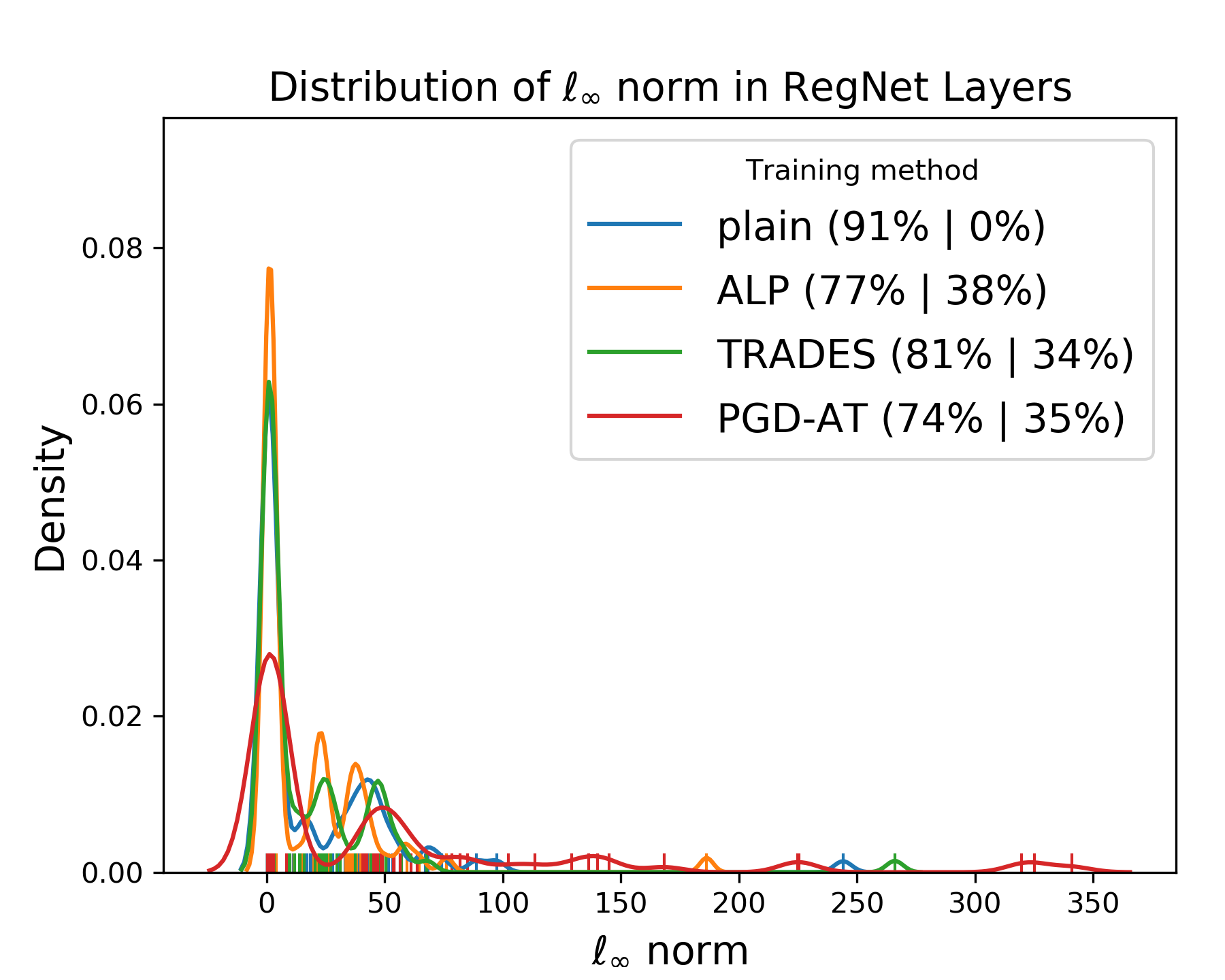}
	\end{center}
	\caption{Comparison of the distribution of norms of the layers of four CNN architectures trained with different adversarial training methods on CIFAR-10. The density is fitted using Gaussian kernel density estimation. The small bars on the bottom of the plots indicate the values of the norms. The two numbers beside each training method are the clean accuracy and robust accuracy, respectively. The robust accuracy is evaluated with standard Auto Attack \cite{croce2020reliable} under $\ell_\infty$ metric with $\epsilon=8/255$.}
	\label{fig:norm}
\end{figure*}

\section{Conclusion and Future Work}
In this paper, we theoretically characterize \norms of convolutional layers and present efficient approaches for computing the exact norms. Our methods are extremely efficient among the existing methods for computing norms of convolutional layers. 
We present norm decay, a novel regularization method, which can improve generalization of CNNs. 
We \emph{prove} that robust classifiers can be realized with neural networks -- a piece of encouraging news to the deep learning community. 

We theoretically analyze the relationship between global Lipschitzness, local Lipschitzness, and the norms of layers. In particular, we show that large norms of layers do not necessarily lead to a large global Lipschitz constant and a large global Lipschitz constant does not necessarily incur small robust accuracy. 
In the experiments, we find that regularizing the norms may not improve adversarial robustness and may even slightly hurt adversarial robustness. 
Moreover, CNNs trained with adversarial training frameworks actually have comparable and even larger layer norms than their non-adversarially robust counterparts, which shows that large norms of layers do not matter. 
Our theoretical result (Proposition~\ref{thm:large}) also suggests that imposing local Lipschitzness on neural nets may be an effective approach in adversarial training, which sheds light on future research. 

\section*{Acknowledgments}
This work was supported by the NSFC under Grant 61976097.

\bibliography{aaai2020}

\clearpage

\appendix
\setcounter{secnumdepth}{2} %

\section*{Appendix}
\section{Proof}
\subsection{Proof of Lemma 1}
For clarity, let $C \in \mathbb{R}^{d_{in} \times h_{in} \times w_{in}}$ denote the 3D input channels of Conv2d. In the vectorization of $C$, We first vectorize $C_{1,:,:}$ (as shown in Figure~\ref{fig:input channel}), and then vectorize $C_{2,:,:}$, and so on, which determines the order of input elements in the vectorized input vector $x$.

\begin{figure*}[!thb]
	\centering
	\begin{subfigure}{0.25\textwidth}
		\centering
		\begin{tikzpicture}[scale=0.3, every node/.style={scale=0.8}]
			\draw[black] (0,0) rectangle (9,9);
			\draw[black] (1,1) rectangle (8,8);
			\draw[step=1,black] (0,4) grid (5,9);
			\foreach \t in {0.5,2.5,4.5}
			{
				\foreach \r in {4.5,6.5,8.5}
				{
					\node[fill=yellow,minimum size=0.33cm] at (\t, \r) {};
				}
			}
			\node at (4.5, 4.5) {$\blacktriangle$};
			\draw[blue,thick] (1,1) rectangle (8,8);
			\draw[decoration={calligraphic brace,amplitude=3pt}, decorate, line width=0.8pt] (5, 9.1) -- (9, 9.1);
			\node at (7, 10) {$w_{in} + 2p_2 - k_2$};
			\draw[decoration={calligraphic brace,amplitude=3pt,mirror,aspect=0.75}, decorate, line width=0.8pt] (9.1, 0) -- (9.1, 4);
			\node at (10.3, 3) {\rotatebox{-90}{$h_{in} + 2p_1 - k_1$}};
		\end{tikzpicture}
		\caption{Kernel at position $P_{c_1 s_1-a_{\max}+a_{\midd}, c_2 s_2-b_{\max}+b_{\midd}}$; $K_{i,j,a_{\max},b_{\max}}$ is multiplied by $\blacktriangle$.}
		\label{fig:kernel1}
	\end{subfigure}
	\hfill
	\begin{subfigure}{0.25\textwidth}
		\centering
		\begin{tikzpicture}[scale=0.3, every node/.style={scale=0.8}]
			\draw[black] (0,0) rectangle (9,9);
			\draw[black] (1,1) rectangle (8,8);
			\draw[step=1,black] (2,2) grid (7,7);
			\foreach \t in {2.5,4.5,6.5}
			{
				\foreach \r in {2.5,4.5,6.5}
				{
					\node[fill=yellow,minimum size=0.33cm] at (\t, \r) {};
				}
			}
			\node at (4.5, 4.5) {$\blacktriangle$};
			\draw[blue,thick] (1,1) rectangle (8,8);
			\draw[decoration={calligraphic brace,amplitude=3pt}, decorate, line width=0.8pt] (-0.1, 7) -- (-0.1, 9);
			\draw[decoration={calligraphic brace,amplitude=3pt}, decorate, line width=0.8pt] (0, 9.1) -- (2, 9.1);
			\draw[decoration={calligraphic brace,amplitude=2pt,mirror}, decorate, line width=0.4pt] (9.1, 8) -- (9.1, 9);
			\draw[decoration={calligraphic brace,amplitude=2pt}, decorate, line width=0.4pt] (8, 9.1) -- (9, 9.1);
			\node at (-1.3, 8) {$c_1 s_1$};
			\node at (1, 9.8) {$c_2 s_2$};
			\node at (10.1, 8.5) {$p_1$};
			\node at (8.5, 9.8) {$p_2$};
		\end{tikzpicture}
		\caption{Kernel at position $P_{c_1 s_1, c_2 s_2}$; $K_{i,j,a_{\midd},b_{\midd}}$ is multiplied by $\blacktriangle$.}
		\label{fig:kernel2}
	\end{subfigure}
	\hfill
	\begin{subfigure}{0.25\textwidth}
		\centering
		\begin{tikzpicture}[scale=0.3, every node/.style={scale=0.8}]
			\draw[black] (0,0) rectangle (9,9);
			\draw[black] (1,1) rectangle (8,8);
			\draw[step=1,black] (4,0) grid (9,5);
			\foreach \t in {4.5,6.5,8.5}
			{
				\foreach \r in {0.5,2.5,4.5}
				{
					\node[fill=yellow,minimum size=0.33cm] at (\t, \r) {};
				}
			}
			\node at (4.5, 4.5) {$\blacktriangle$};
			\draw[blue,thick] (1,1) rectangle (8,8);
			\draw[decoration={calligraphic brace,amplitude=3pt}, decorate, line width=0.8pt] (1, 9.1) -- (8, 9.1);
			\node at (4.5, 10) {$w_{in}$};
			\draw[decoration={calligraphic brace,amplitude=3pt}, decorate, line width=0.8pt] (-0.1, 1) -- (-0.1, 8);
			\node at (-1.3, 4.5) {$h_{in}$};
		\end{tikzpicture}
		\caption{Kernel at position $P_{c_1 s_1+a_{\midd}-a_{\min}, c_2 s_2+b_{\midd}-b_{\min}}$; $K_{i,j,a_{\min},b_{\min}}$ is multiplied by $\blacktriangle$.}
		\label{fig:kernel3}
	\end{subfigure}
	\caption{An illustration of the proof of Lemma~1. The blue rectangle is the input channel \emph{excluding padding}, which is of size $7 \times 7$, i.e., $h_{in}=w_{in}=7$. Outside the blue rectangle is padding ($p_1=p_2=1$). The grids are kernel elements (kernel size $k_1=k_2=5$). Strides are $s_1=s_2=2$. The indices set of the yellow kernel elements is $\mathcal{A}_{(a_{\min},b_{\min})} \in \mathcal{S}$ where $a_{\min}=b_{\min}=1$. All yellow kernel elements are multiplied by $\blacktriangle$ during convolution, which indicates $\mathcal{G}_{ij} \supseteq \mathcal{K}_{ij}$.}
	\label{fig:lemma1}
\end{figure*}
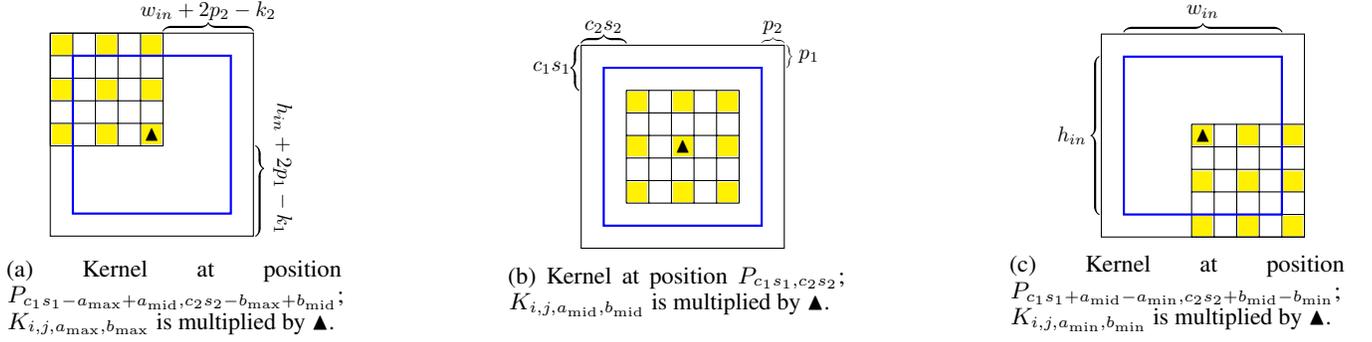

\begin{proof}[Proof of Lemma~\ref{lem:kernel}]
	The vectorized convolution for $\conv$ is $\conv(x) = Mx$, where $x$ is the vectorization of input channels and the length of $x$ is $d_{in}h_{in}w_{in}$. By inspecting the convolution operation, we first note that every nonzero element in $M$ is a kernel element in $K$. 
	By the matrix-vector multiplication $Mx = \sum_{n=1}^{d_{in} h_{in} w_{in}} x_{n} M_{:,n}$, we have, for every $n = 1, \dots, d_{in} h_{in} w_{in}$, $M_{:,n}$ is multiplied by $x_{n}$, which is an element in the $j$-th input channel $C_{j,:,:}$, where $j = \lceil n / (h_{in} w_{in}) \rceil$. %
	For every $i = 1, \dots, d_{out}$, the kernel slice $K_{i,j,:,:}$ is convolved with $C_{j,:,:}$ that contains $x_{n}$. 
	Let $\mathcal{G}_{i}$ be the set of kernel elements in $K_{i,j,:,:}$ that are multiplied by $x_{n}$. Let $(a, b)$ be the smallest indices\footnote{Both $a$ and $b$ are the smallest.} such that $K_{i,j,a,b}\in \mathcal{G}_{i}$. For all $K_{i,j,c,d} \in \mathcal{G}_{i}$, $(c, d)$ must satisfy $(a - c) \equiv 0 \Mod{s_1}$ and $(b - d) \equiv 0 \Mod{s_2}$, i.e., their vertical (resp. horizontal) distance must be a multiple of the vertical (resp. horizontal) stride of $\conv$. Then $(c,d) \sim (a,b)$. Besides, the total vertical (resp. horizontal) distance the kernel can possibly shift on the input channel $C_{j,:,:}$ must be smaller than $h_{in} + 2p_1 - k_1$ (resp. $w_{in} + 2p_2 - k_2$) (see Figure~\ref{fig:kernel1}). Thus $(c, d)$ must satisfy $0 \leq c-a \leq h_{in} + 2p_1 - k_1$ and $0 \leq d-b \leq w_{in} + 2p_2 - k_2$. 
	Therefore, by the construction of $\mathcal{S}$, $\mathcal{G}_{i} \subseteq \{K_{i,j, k, t} \colon (k, t) \in \mathcal{A}_{(a,b)}\}$ where $\mathcal{A}_{(a,b)} \in \mathcal{S}$. Then $\nz(M_{:,n}) = \cup_{i=1}^{d_{out}} \mathcal{G}_{i} \subseteq \{K_{i,j, k, t} \colon 1 \leq i \leq d_{out}, (k, t) \in \mathcal{A}_{(a,b)}\}$, which proves the first claim in Lemma~1.
	
	In the following proof, for a kernel slice\footnote{For simplicity, a kernel slice is referred to as kernel in the text that follows.} $K_{i,j,:,:}$, we use the coordinates of its upper left corner on the input channel $C_{j,:,:}$ to indicate its position. For example, at the beginning of convolution, the kernel is at position $P_{0,0}$. Note that the coordinates of kernel are always multiples of strides.
	By Assumption~1, we have $k_1 + c_1 s_1 - p_1 \leq h_{in}$ and $k_2 + c_2 s_2 - p_2 \leq w_{in}$. Then $P_{c_1 s_1, c_2 s_2}$ is a legitimate kernel position.\footnote{By legitimate kernel position, we mean the kernel is within the boundary of input channels (including padding, if any) and the coordinates of kernel are multiples of strides.} At position $P_{c_1 s_1, c_2 s_2}$, all kernel elements are multiplied by some input elements \emph{but not padding elements} (see Figure~\ref{fig:kernel2}).
	For any $\mathcal{A} \in \mathcal{S}$, let $a_{\max} = \max \{a \colon (a,b) \in \mathcal{A}\}$ and $b_{\max} = \max \{b \colon (a,b) \in \mathcal{A}\}$, and let $a_{\min} = \min \{a \colon (a,b) \in \mathcal{A}\}$ and $b_{\min} = \min \{b \colon (a,b) \in \mathcal{A}\}$. Let $r_1$ and $r_2$ be the largest integers such that $r_1 s_1 \leq h_{in} + 2p_1 - k_1$ and $r_2 s_2 \leq w_{in} + 2p_2 - k_2$. By the definition of $\mathcal{S}$, we have $a_{\max}-a_{\min} \leq r_1 s_1$ and $b_{\max}-b_{\min} \leq r_2 s_2$.
	Let $a_{\midd}=\max(a_{\min}, a_{\max}-c_1 s_1)$ and $b_{\midd}=\max(b_{\min}, b_{\max}-c_2 s_2)$. Then we have $(a_{\midd}, b_{\midd}) \in \mathcal{A}$ because $(a_{\midd}, b_{\midd}) \sim (a_{\max}, b_{\max})$ and $0 \leq a_{\midd} - a_{\min} \leq h_{in} + 2p_1 - k_1$ and $0 \leq b_{\midd} - b_{\min} \leq w_{in} + 2p_2 - k_2$. Suppose when the kernel is at position $P_{c_1 s_1, c_2 s_2}$, for any $i$ such that $1 \leq i \leq d_{out}$ and any $j$ such that $1 \leq j \leq d_{in}$, the kernel element $K_{i,j,a_{\midd},b_{\midd}}$ is multiplied by the element $\blacktriangle$ on the $j$-th input channel.\footnote{Assumption 1 ensures that $\blacktriangle$ is indeed an input element instead of a padding element, and thus $\blacktriangle$ is an element of $x$.} Then when the kernel is at position $P_{c_1 s_1-a_{\max}+a_{\midd}, c_2 s_2-b_{\max}+b_{\midd}}$, the kernel element $K_{i,j,a_{\max},b_{\max}}$ is multiplied by $\blacktriangle$. And when the kernel is at position $P_{c_1 s_1+a_{\midd}-a_{\min}, c_2 s_2+b_{\midd}-b_{\min}}$, the kernel element $K_{i,j,a_{\min},b_{\min}}$ is multiplied by $\blacktriangle$. 
	To show the last two claims are true, we need to show $P_{c_1 s_1-a_{\max}+a_{\midd}, c_2 s_2-b_{\max}+b_{\midd}}$ is a legitimate kernel position. We note that 
	\begin{equation}
		\begin{gathered}
			c_1 s_1-a_{\max}+a_{\midd} = \\ 
			\begin{cases}
				c_1 s_1-a_{\max}+a_{\min} \geq 0 & \mbox{if } a_{\min} \geq a_{\max}-c_1 s_1\\
				0 & \mbox{if } a_{\min} < a_{\max}-c_1 s_1
			\end{cases}
		\end{gathered}
	\end{equation}
	which shows that $c_1 s_1-a_{\max}+a_{\midd} \geq 0$ and is a multiple of stride $s_1$. Similarly, $c_2 s_2-b_{\max}+b_{\midd} \geq 0$ and is a multiple of stride $s_2$. Thus $P_{c_1 s_1-a_{\max}+a_{\midd}, c_2 s_2-b_{\max}+b_{\midd}}$ is a legitimate kernel position.
	To see $P_{c_1 s_1+a_{\midd}-a_{\min}, c_2 s_2+b_{\midd}-b_{\min}}$ is a legitimate kernel position, we note that 
	\begin{equation}
		\begin{gathered}
			c_1 s_1+a_{\midd}-a_{\min} = \\
			\begin{cases}
				c_1 s_1 & \mbox{if } a_{\min} \geq a_{\max}-c_1 s_1\\
				a_{\max}-a_{\min} \leq r_1 s_1 & \mbox{if } a_{\min} < a_{\max}-c_1 s_1
			\end{cases}
		\end{gathered}
	\end{equation}
	Similarly, $ c_2 s_2+b_{\midd}-b_{\min} = c_2 s_2$ or $=b_{\max}-b_{\min} \leq r_2 s_2$.
	Since $P_{c_1 s_1, c_2 s_2}$, $P_{c_1 s_1, r_2 s_2}$, $P_{r_1 s_1, c_2 s_2}$, and $P_{r_1 s_1, r_2 s_2}$ are legitimate kernel positions, $P_{c_1 s_1+a_{\midd}-a_{\min}, c_2 s_2+b_{\midd}-b_{\min}}$ is also a legitimate kernel position.
	
	Since both $K_{i,j,a_{\max},b_{\max}}$ and $K_{i,j,a_{\min},b_{\min}}$ are multiplied by $\blacktriangle$, then for all $(c,d) \in \mathcal{A}$, $K_{i,j,c,d}$ is multiplied by $\blacktriangle$. Let $\mathcal{G}_{ij}$ be the set of kernel elements in $K_{i,j,:,:}$ that are multiplied by $\blacktriangle$ and let $\mathcal{K}_{ij} \ceq \{K_{i,j, k, t} \colon (k, t) \in \mathcal{A}\}$. Then $\mathcal{G}_{ij} \supseteq \mathcal{K}_{ij}$. Note that this is true for all $i$ such that $1 \leq i \leq d_{out}$ and all $j$ such that $1 \leq j \leq d_{in}$. Let $M_{:,n}$ be the column of $M$ such that $M_{:,n}$ is multiplied by $\blacktriangle$ in $\conv(x)=Mx$. Clearly, $(j-1)h_{in}w_{in} < n \leq j h_{in}w_{in}$. Recall that $\nz(M_{:,n})$ is the set of kernel elements that are multiplied by $\blacktriangle$. And note that for the 2D multi-channel convolution $\conv$, $\blacktriangle$ is convolved with kernel slices $K_{i,j,:,:}$ for all $i$ such that $1 \leq i \leq d_{out}$. Then $\nz(M_{:,n}) = \cup_{i=1}^{d_{out}}\mathcal{G}_{ij} \supseteq \cup_{i=1}^{d_{out}}\mathcal{K}_{ij} = \{K_{i,j, k, t} \colon 1 \leq i \leq d_{out}, (k, t) \in \mathcal{A}\}$, which completes the proof.
\end{proof}

\subsection{Proof of Theorem 1}
\begin{proof}[Proof of Theorem~\ref{thm:norm}]
	Let $\mathcal{F} \ceq \{\nz(M_{:,n}) \colon 1 \leq n \leq d_{in}h_{in}w_{in}\}$ and $\mathcal{T}_{j}^{\mathcal{A}} \ceq \{K_{i,j, k, t} \colon 1 \leq i \leq d_{out}, (k, t) \in \mathcal{A}\}$, and let $\mathcal{H} \ceq \{\mathcal{T}_{j}^{\mathcal{A}} \colon 1\leq j \leq d_{in}, \mathcal{A} \in \mathcal{S}\}$. Define a function $\abs$ from sets of real numbers to non-negative numbers $\abs \colon \mathcal{C} \mapsto \sum_{c \in \mathcal{C}} |c|$. Let $\mathcal{R} \ceq \mathcal{F} \cup \mathcal{H}$ and $\mathcal{W} \ceq \{\abs(\mathcal{C}) \colon \mathcal{C} \in \mathcal{R}\}$. Then $\mathcal{W}$ is bounded above by $\abs(\set(K))$, where $\set(K)$ is the set of all elements of 4D kernel $K$, as we now explain. For every $\mathcal{C} \in \mathcal{F}$, by Lemma~1 we have $\mathcal{C} \subseteq \mathcal{B}$ for some $\mathcal{B} \in \mathcal{H}$, and thus $\abs(\mathcal{C}) \leq \abs(\mathcal{B})$. But for every $\mathcal{B} \in \mathcal{H}$, $\mathcal{B} \subseteq \set(K)$ and thus $\abs(\mathcal{B}) \leq \abs(\set(K))$. Then $\abs(\mathcal{C}) \leq \abs(\set(K))$, which proves the last claim. Since $\mathcal{W}$ is a finite set, $\max \mathcal{W} = \sup \mathcal{W} < \infty$. Then there exists a set $\mathcal{C} \in \mathcal{R}$ such that $\abs(\mathcal{C}) = \max \mathcal{W}$. 
	Suppose $\mathcal{C} \in \mathcal{F}$. Then by Lemma~1 there exists $\mathcal{B} \in \mathcal{H}$ such that $\mathcal{C} \subseteq \mathcal{B}$, and thus $\abs(\mathcal{C}) \leq \abs(\mathcal{B})$. However, since $\abs(\mathcal{C}) = \max \mathcal{W}$, we also have $\abs(\mathcal{C}) \geq \abs(\mathcal{B})$. Thus $\abs(\mathcal{C}) = \abs(\mathcal{B})$. 
	On the other hand, suppose $\mathcal{C} \in \mathcal{H}$. Then by Lemma~1 there exists $\mathcal{B} \in \mathcal{F}$ such that $\mathcal{C} \subseteq \mathcal{B}$, and thus $\abs(\mathcal{C}) \leq \abs(\mathcal{B})$. However, since $\abs(\mathcal{C}) = \max \mathcal{W}$, we also have $\abs(\mathcal{C}) \geq \abs(\mathcal{B})$. Thus $\abs(\mathcal{C}) = \abs(\mathcal{B})$. 
	The last two results show that there are always a pair of sets $\mathcal{C} \in \mathcal{H}$ and $\mathcal{B} \in \mathcal{F}$ such that $\abs(\mathcal{C}) = \abs(\mathcal{B}) = \max \mathcal{W}$. Then $\| \conv \|_{1} = \|M\|_{1} = \max_{n} \abs(\nz(M_{:,n})) = \abs(\mathcal{B}) = \abs(\mathcal{C}) = \max_{1 \leq j \leq d_{in}} \max_{\mathcal{A} \in \mathcal{S}} \sum_{(k,t) \in \mathcal{A}} \sum_{i=1}^{d_{out}} | K_{i,j,k,t} |$.
	
	Let $y=\conv(x) = Mx$. Then $y_{n} = \langle M_{n,:}, x \rangle$. We note that, for all elements $y_{n}$ on output channels, $y_{n}$ is also the result of a kernel slice $K_{k,:,:,:}$ being convolved with a part of the input channels $C_{:, i:i+k_1, j:j+k_2}$, where $k=\lceil n/(h_{out}w_{out}) \rceil$ and $C$ is the input channels including padding.
	By Assumption~1, when $i=c_1 s_1$ and $j=c_2s_2$, $\set(C_{:, i:i+k_1, j:j+k_2}) \subseteq \set(D)$ where $D$ is the input channels excluding padding (see Figure~\ref{fig:kernel2} where the blue rectangle is a slice of $D$). In this case, it is clear that $\nz(M_{n,:}) = \set(K_{k,:,:,:})$. 
	When the part of input channels $C_{:, i:i+k_1, j:j+k_2}$ being convolved with $K_{k,:,:,:}$ includes padding, $\nz(M_{n,:}) \subset \set(K_{k,:,:,:})$, because $x = \vect(D)$ does not include padding elements (see the matrix in Figure~1 for an illustration). 
	Note that, as convolution produces output elements $y_{n}$ one by one, it iterates all kernel slices $K_{k,:,:,:}$ for $k$ in the range $[1, d_{out}]$. Thus, $\max_{n} \abs(\nz(M_{n,:})) = \max_{1 \leq k \leq d_{out}} \abs(\set(K_{k,:,:,:}))$.
	Then, $\| \conv \|_{\infty} = \|M\|_{\infty} = \max_{n} \abs(\nz(M_{n,:})) = \max_{1 \leq k \leq d_{out}} \abs(\set(K_{k,:,:,:})) = \max_{1 \leq i \leq d_{out}} \sum_{j=1}^{d_{in}} \sum_{k=1}^{k_1} \sum_{t=1}^{k_2} | K_{i,j,k,t} |$.
	
	By the result we have just obtained, for every output element $y_{n} = \langle M_{n,:}, x \rangle$, $\nz(M_{n,:}) = \set(K_{k,:,:,:})$ or $\nz(M_{n,:}) \subset \set(K_{k,:,:,:})$. And for a fixed $k$ such that $1 \leq k \leq d_{out}$, $K_{k,:,:,:}$ performs exactly $h_{out}w_{out}$ times convolution to produce $h_{out}w_{out}$ elements on the output channels (see the matrix in Figure~\ref{fig:conv} for an illustration). Therefore,
	\begin{align}
		\|M\|_{\mathrm{F}} &=\bigg(\sum_{n=1}^{d_{out}h_{out}w_{out}} \sum\big\{t^2\colon t\in \nz(M_{n,:})\big\}\bigg)^{\frac{1}{2}}\\ 
		\leq &\bigg(\sum_{i=1}^{d_{out}} h_{out} w_{out} \sum \big\{t^2 \colon t \in \set(K_{i,:,:,:})\big\}\bigg)^{\frac{1}{2}}\\ 
		= &\bigg(h_{out} w_{out} \sum_{i=1}^{d_{out}} \sum_{j=1}^{d_{in}} \sum_{k=1}^{k_1} \sum_{t=1}^{k_2} | K_{i,j,k,t} |^2\bigg)^{\frac{1}{2}}
	\end{align}
	The fact that $\|\conv\|_{2} = \|M\|_{2} \leq \|M\|_{\mathrm{F}}$ completes the proof.
\end{proof}

\subsection{Some Remarks of Theorem 1}
\begin{remark}
	Following the methods in the proof of Theorem~\ref{thm:norm}, we can compute $\|M\|_{\mathrm{F}}$ exactly, though the formula for $\|M\|_{\mathrm{F}}$ might be complicated.
\end{remark}
\begin{remark}
	If there is no padding, then for all $n$ such that $1 \leq n \leq d_{out}h_{out}w_{out}$, $\nz(M_{n,:}) = \set(K_{k,:,:,:})$ for some $k$. Then we have $\|M\|_{\mathrm{F}} = (h_{out} w_{out} \sum_{i=1}^{d_{out}} \sum_{j=1}^{d_{in}} \sum_{k=1}^{k_1} \sum_{t=1}^{k_2} | K_{i,j,k,t} |^2)^{\frac{1}{2}}$.
	Besides, it is possible that $\|M\|_{2} = \|M\|_{\mathrm{F}}$. If the two conditions hold, the bound for the $\ell_2$ norm is sharp: $\|\conv\|_{2} = (h_{out} w_{out} \sum_{i=1}^{d_{out}} \sum_{j=1}^{d_{in}} \sum_{k=1}^{k_1} \sum_{t=1}^{k_2} | K_{i,j,k,t} |^2)^{\frac{1}{2}}$.
\end{remark}

\subsection{Proof of Proposition~1}
\begin{proof}
	Consider an $L$-layer feedforward network with ReLU activation (denoted by $\sigma(\cdot)$) where the weight matrices of all layers are diagonal matrices (without bias for simplicity) and denote the diagonal of the weight matrix of the $i$-th layer as $\bv{d}_i$. In the network there are two consecutive layers where $\bv{d}_j \odot \bv{d}_{j+1} = \bv{0}$, where $\odot$ denotes element-wise multiplication. Denote the input of the $j$-th layer as $\bv{x}_{j-1}$. Then the output of $j$-th layer is $\bv{x}_{j}=\sigma(\bv{x}_{j-1} \odot \bv{d}_j )$. And $\bv{x}_{j+1}=\sigma(\bv{x}_{j} \odot \bv{d}_{j+1} )$. For any input $\bv{x}_{0}\in\mathbb{R}^n$, we have $\bv{x}_{j+1} \equiv \bv{0}$. Thus, the output of the entire network is always $\bv{0}$, which means its Lipschitz constant is 0. Since for all $i\in[L]$, at least one element in $\bv{d}_i$ can be arbitrarily large, then the norm of each layer can be arbitrarily large, which completes the proof.
\end{proof}

\subsection{Proof of Theorem~2}
\begin{proof}
	Without loss of generality, assume $\mc{X} \subset [0, 1]^n$. 
	Define the robust cover for class~$c$ as $\mv{R}{c} \cq \cup_{x \in \mv{X}{c}} \ball(x, \epsilon)$. Let $\cl(\mv{R}{c})$ be the closure of $\mv{R}{c}$. Then we can show that for every pair of two classes $i \neq j$, $\cl(\mv{R}{i}) \cap \cl(\mv{R}{j}) = \emptyset$. First note that $\mv{R}{i} \cap \mv{R}{j} = \emptyset$ by Assumption~\ref{assum:separate}. Then it suffices to show that for every limit point $p$ of $\mv{R}{i}$, $p \notin \cl(\mv{R}{j})$, which will be proved by contradiction. 
	Suppose $p$ is a limit point of $\mv{R}{i}$ and $p \in \cl(\mv{R}{j})$. Then for all $\xi > 0$, there exist $r^{(i)} \in \mv{R}{i}$ and $r^{(j)} \in \mv{R}{j}$ such that $d(p, r^{(i)}) \leq \xi$ and $d(p, r^{(j)}) \leq \xi$. Since there exist $x^{(i)} \in \mv{X}{i}$ and $x^{(j)} \in \mv{X}{j}$ such that $d(x^{(i)}, r^{(i)}) \leq \epsilon$ and $d(x^{(j)}, r^{(j)}) \leq \epsilon$. By triangle inequality, we have $d(x^{(i)}, p) \leq \xi + \epsilon$ and $d(x^{(j)}, p) \leq \xi + \epsilon$. Using triangle inequality again, we have $d(x^{(i)}, x^{(j)}) \leq 2(\xi + \epsilon)$. Then $\sup d(x^{(i)}, x^{(j)}) \to 2 \epsilon$ as $\xi \to 0$, which implies $\inf\{d(x^{(i)}, x^{(j)}) \colon x^{(i)} \in \mv X i, x^{(j)} \in \mv X j\} \leq 2\epsilon$. This contradicts Assumption~\ref{assum:separate} and thus for every limit point $p$ of $\mv{X}{i}$, $p \notin \cl(\mv{R}{j})$. By symmetry, we have for every limit point $p$ of $\mv{X}{j}$, $p \notin \cl(\mv{R}{i})$, Thus, $\cl(\mv{R}{i}) \cap \cl(\mv{R}{j}) = \emptyset$ for any $i \neq j$. 
	
	Let $\mathrm{1}_{\cl(\mv{R}{c})}$ be the indicator function of the set $\cl(\mv{R}{c})$, i.e., $\mathrm{1}_{\cl(\mv{R}{c})}(x)=1$ if $x \in \cl(\mv{R}{c})$ else $=0$. Let $\overline{\mc X}=\cup_{c=1}^{C}\cl(\mv{R}{c})$ and define a function $h \colon \overline{\mc X} \to \mc{Y}$ by $h(x) = \sum_{c=1}^{C} c \cdot \mathrm{1}_{\cl(\mv{R}{c})}(x)$. Note that $h$ can correctly predict the labels of points in the set $\overline{\mc X}$ and thus have robust accuracy 1. 
	
	We now show that $h$ is  continuous on $\overline{\mc{X}}$.
	For all $x \in \overline{\mc{X}}$, we have $x \in \cl(\mv{R}{c})$ for some c. Then there exits $\delta>0$ such that $\ball(x, \delta) \cap \cl(\mv{R}{j}) = \emptyset$ for all $j \neq c$ (because otherwise $x$ would be a limit point of $\cl(\mv{R}{j})$ and $\cl(\mv{R}{j}) \cap \cl(\mv{R}{c}) \neq \emptyset$. Let $V= \ball(x, \delta) \cap \overline{\mc{X}}$ then $V \subset \cl(\mv{R}{c})$. Thus for all $s\in V$, $|h(x)-h(s)|=|c-c|=0<\epsilon$ for all $\epsilon>0$. Thus $h$ is continuous on $\overline{\mc{X}}$.
	
	Note that $\overline{\mc X}$ is closed and bounded and thus compact, and $h$ is continuous on $\overline{\mc X}$, i.e., $h \in C(\overline{\mc X}, \mathbb{R})$. Then, by the Universal Approximation Theorem (Theorem 3.2) in \cite{kidger2020universal}, for all $\zeta > 0$, there exists a feedforward neural network $F\colon \overline{\mc X} \to \mathbb{R}$ with $\rho$ being the activation function such that $\sup_{x \in \overline{\mc X}} |F(x)-h(x)| \leq \zeta$. Let $\zeta = 0.1$. Then the robust accuracy of the neural network $F$ is just the robust accuracy of $h$, which is $\mc{D}(\mc{X})=1$, which is the desired result.
\end{proof}

\subsection{Proof of Proposition~2}
\begin{proof}
	We use the notations in the proof of Theorem~\ref{thm:realizability}. 
	Without loss of generality, assume $\overline{\mc X} \subset [0, 1]^n$. 
	Let $\tilde{\mc X} = \overline{\mc X} \cup [2,3]^n \cup [u, 4]^n$, where $u \in (3,4)$. Then $\overline{\mc X} \cap [2,3]^n = \emptyset$ and $\overline{\mc X} \cap [u, 4]^n = \emptyset$ and $[2,3]^n \cap [u, 4]^n= \emptyset$, and thus $\tilde{\mc X}$ is compact. Let $h(x) = \sum_{c=1}^{C} c \cdot \mathrm{1}_{\cl(\mv{R}{c})}(x) - 1.1\cdot\mathrm{1}_{[2,3]^n}(x) + 0.1\cdot\mathrm{1}_{[u, 4]^n}(x)$. Let $\zeta = 0.1$. Then following the same argument in the proof of Theorem~\ref{thm:realizability}, $h$ is continuous on $\overline{\mc{X}}$, i.e., $h \in  C(\tilde{\mc X}, \mathbb{R})$. Then there exists a feedforward neural network $F$ with $\rho$ being the activation function that achieves robust accuracy $1$. 
	Moreover, since $F$ is composition of continuous functions, $F$ is always continuous on $\mathbb{R}^n$, which allows us to analyze the Lipschitz property of $F$. 
	Consider two points $t_1 = (3, 3, \dots, 3) \in [2,3]^n$ and $t_2 = (u, u, \dots, u)\in [u, 4]^n$. The Lipschitz constant of the neural network $F$ has a lower bound $L=|F(t_1)-F(t_2)|/\|t_1-t_2\|\geq 1 /\|t_1-t_2\|$. As $u \to 3$, $\|t_1-t_2\| \to 0$ and thus $L \to \infty$, which completes the proof. 
\end{proof}

\begin{table*}[bhtp]
	\centering
	\resizebox{\textwidth}{!}{%
		\setlength{\tabcolsep}{2.5pt}
		\begin{tabular}{c|c|c|ccccc|ccccc|ccccc|ccccc}
			\toprule
			\multicolumn{1}{c}{} & & plain & \multicolumn{5}{c|}{weight decay} & \multicolumn{5}{c|}{singular value clipping} & \multicolumn{5}{c|}{$\ell_1$ norm decay} & \multicolumn{5}{c}{$\ell_\infty$ norm decay} \\
			\midrule
			model & ACC & --- & $10^{-1}$ & $10^{-2}$ & $10^{-3}$ & $10^{-4}$ & $10^{-5}$ & 0.1 & 0.5 & 1.0 & 1.5 & 2.0 & $10^{-1}$ & $10^{-2}$ & $10^{-3}$ & $10^{-4}$ & $10^{-5}$ & $10^{-1}$ & $10^{-2}$ & $10^{-3}$ & $10^{-4}$ & $10^{-5}$ \\
			\midrule
			\multirow{2}[0]{*}{vgg}&
			Clean & 90.4 & 52.9 & 91.6 & 91.7 & 90.1 & 90.2 & 85.9 & 87.6 & 89.1 & 90.0 & 89.9 & 77.0 & 88.1 & 91.1 & 90.6 & 90.8 & 86.3 & 91.8 & 91.1 & 90.8 & 90.6
			\\
			& Robust & 60.2 & 17.3 & 56.3 & 60.5 & 60.6 & 60.3 & 55.1 & 48.8 & 52.2 & 54.1 & 56.7 & 55.8 & 56.5 & 62.5 & 61.1 & 60.1 & 47.7 & 56.9 & 60.0 & 60.8 & 60.1
			\\
			\midrule
			\multirow{2}[0]{*}{resnet}&
			Clean & 93.2 & 53.4 & 94.3 & 94.1 & 93.1 & 92.7 & 91.9 & 93.6 & 94.0 & 94.2 & 93.8 & 84.2 & 92.5 & 93.4 & 93.5 & 93.4 & 85.9 & 93.0 & 93.8 & 93.1 & 93.0
			\\
			& Robust & 37.0 & 14.5 & 28.2 & 33.7 & 33.9 & 40.9 & 44.1 & 35.2 & 41.7 & 43.2 & 39.8 & 35.9 & 24.5 & 37.7 & 38.3 & 37.5 & 25.8 & 20.0 & 34.7 & 38.9 & 37.6
			\\
			\midrule
			\multirow{2}[0]{*}{senet}&
			Clean & 93.1 & 10.0 & 94.2 & 93.9 & 93.0 & 92.4 & 90.4 & 93.8 & 94.2 & 93.8 & 94.2 & 78.0 & 92.3 & 93.8 & 93.3 & 93.3 & 86.9 & 93.0 & 93.6 & 92.8 & 93.2
			\\
			& Robust & 35.7 & 10.0 & 23.5 & 32.8 & 37.0 & 34.8 & 31.6 & 30.5 & 35.6 & 35.2 & 37.4 & 42.3 & 33.6 & 36.0 & 38.2 & 36.7 & 36.8 & 28.6 & 31.0 & 37.6 & 37.4
			\\
			\midrule
			\multirow{2}[0]{*}{regnet}&
			Clean & 91.8 & 18.8 & 93.6 & 94.4 & 92.3 & 91.3 & 91.8 & 93.9 & 93.4 & 93.0 & 92.4 & 88.2 & 93.7 & 92.3 & 91.6 & 91.9 & 87.9 & 93.4 & 92.0 & 91.8 & 91.9
			\\
			& Robust & 34.8 & 15.0 & 23.7 & 30.3 & 30.0 & 31.0 & 28.5 & 27.7 & 28.8 & 29.0 & 28.8 & 15.6 & 29.2 & 31.1 & 28.1 & 34.3 & 15.6 & 23.2 & 27.7 & 27.9 & 30.6
			\\
			\bottomrule
	\end{tabular}}
	\caption{Comparison of clean accuracy (\%) and robust accuracy (\%) of 4 CNN models trained with different norm-regularization methods on CIFAR-10. The second row corresponds to the values of regularization parameters. Robust accuracy is evaluated with standard Auto Attack \ccite{croce2020reliable} under $\ell_\infty$ metric at $\epsilon=1/255$.}
	\label{tab:acc2}
\end{table*}

\begin{table*}[!t]
	\centering
	\begin{tabular}{c|c|c|ccccc|ccccc}
		\toprule
		\multicolumn{1}{c}{} & & plain & \multicolumn{5}{c|}{$\ell_1$ norm decay and projecting BN} & \multicolumn{5}{c}{$\ell_\infty$ norm decay and projecting BN} \\
		\midrule
		model & ACC & --- & $10^{-1}$ & $10^{-2}$ & $10^{-3}$ & $10^{-4}$ & $10^{-5}$ & $10^{-1}$ & $10^{-2}$ & $10^{-3}$ & $10^{-4}$ & $10^{-5}$ \\
		\midrule
		\multirow{2}[0]{*}{vgg}&
		Clean & 90.4 & 10.0 & 85.8 & 90.7 & 90.7 & 90.4 & 19.9 & 91.9 & 91.4 & 90.5 & 90.5 \\
		& Robust & 60.2 & 9.8 & 43.5 & 60.7 & 64.9 & 62.1 & 16.5 & 57.3 & 62.4 & 61.4 & 61.8 \\
		\midrule
		\multirow{2}[0]{*}{resnet}&
		Clean & 93.2 & 10.0 & 10.4 & 93.9 & 93.1 & 93.2 & 10.0 & 92.7 & 94.0 & 93.0 & 93.3 \\
		& Robust & 37.0 & 9.8 & 9.6 & 35.9 & 37.9 & 38.0 & 10.0 & 27.8 & 34.3 & 39.6 & 36.5 \\
		\midrule
		\multirow{2}[0]{*}{senet}&
		Clean & 93.1 & 10.0 & 85.0 & 93.8 & 93.5 & 93.5 & 15.7 & 80.4 & 93.8 & 93.3 & 93.2 \\
		& Robust & 35.7 & 8.3 & 18.9 & 29.7 & 34.1 & 33.4 & 15.1 & 17.8 & 26.5 & 32.6 & 33.7 \\
		\midrule
		\multirow{2}[0]{*}{regnet}&
		Clean & 91.8 & 73.1 & 93.3 & 92.5 & 91.9 & 91.8 & 49.8 & 93.5 & 92.2 & 92.2 & 91.7 \\
		& Robust & 34.8 & 16.1 & 27.6 & 32.4 & 33.4 & 31.5 & 15.3 & 16.8 & 29.2 & 31.2 & 31.1 \\
		\bottomrule
	\end{tabular}
	\caption{Comparison of clean accuracy (\%) and robust accuracy (\%) of 4 CNN models trained with $\ell_1$ and $\ell_\infty$ norm decay (and projecting BN norms to 5) on CIFAR-10. The second row corresponds to the values of the regularization parameter $\beta$. Robust accuracy is evaluated with standard Auto Attack \ccite{croce2020reliable} under $\ell_\infty$ metric at $\epsilon=1/255$.}
	\label{tab:acc3}
\end{table*}

\section{More Experimental Results}
Due to space limitation in the main text, we provide more experimental results here. \textbf{Some figures are omitted here to compress the size of this file.} More figures can be found at \url{https://drive.google.com/file/d/1DxJPy_mDtHejr8bLJmIPawLreqwE7yaT/view?usp=sharing}.

\subsection{Clean and Robust Accuracy of CNNs with Norm-Regularization}
In the experiments, we test the regularization parameter of norm decay (ND) and weight decay (WD) in $\{10^{-5}, \dots, 10^{-1}\}$ and test the parameter of singular value clipping (SVC) in $\{2.0, 1.5, 1.0, 0.5, 0.1\}$, while in the main text the strongest regularization (corresponds to 0.1 for SVC and $10^{-1}$ for ND and WD) is omitted due to space limitation. The complete result is shown in Table~\ref{tab:acc2}. 
Again, we notice that regularization can improve generalization but has little effect on adversarial robustness. 
Besides, regularization that is too strong basically reduces both standard accuracy and robust accuracy.

\subsection{How the Norms Change During Training under Norm-Regularization?}\label{sec:norm-change}
We calculate the $\ell_1$ norm (or the $\ell_\infty$ norm when applying $\ell_\infty$ norm decay) of all layers during the training of CNNs under norm-regularization and plot the results in Figure~5,~6,~and~7. Here, we only show the norms of the ResNet layers since the other three models present similar patterns in the change of the norms. Apart from convolutional and fully connected layers, we also show the norms of batch normalization layers (BN) \ccite{ioffe2015batch}. The batch normalization is applied as follows: 
\begin{equation}
	\hat{x}_{i} = \gamma_{i} \frac{x_{i}-\mu_{i}}{\sigma_{i}} + \beta_{i},
\end{equation}
where $i$ is the index of features, and $\mu_{i}$ and $\sigma_{i}$ are respectively the mean and standard deviation of the $i$-th feature. 
Since $\mu_{i}$ and $\sigma_{i}$ are fixed at inference time, BN is simply an affine transformation and its $\ell_{1}$, $\ell_{2}$, and $\ell_{\infty}$ norms are $\max_{i} \gamma_i / \sigma_{i}$. 

We can see in Figure~5 that all regularization methods can effectively regularize the norms of convolutional and fully connected layers when the regularization parameter is set properly. %
Moreover, the $\ell_1$ norms in SVC (Figure~5 g-k) remain basically the same during training. Since SVC clips the $\ell_2$ norms to a fixed value, it indicates that the $\ell_1$ norm is strongly correlated to the $\ell_2$ norm. %
It shows that our approaches to computing the $\ell_1$ and $\ell_\infty$ norms for convolutional layers are equivalent to computing the $\ell_2$ norms while our methods are much more efficient.  

We notice that in some cases the norms of BN explode, as shown in Figure~6. Since we do not explicitly regularize the norms of BN, it seems that the explosion is compensation for the reduction in the norms of convolutional and fully connected layers. In order to investigate whether regularizing the norms of BN improves adversarial robustness, we further project the $\ell_1$ (also $\ell_\infty$) norms of BN to a fixed value after applying norm decay at each step. The clean and robust accuracy is shown in Table~\ref{tab:acc3} and the change of norms during training is shown in Figure~7. We use projection instead of extending norm decay to BN because we find that norm decay is too ``soft'' to regularize the norms of BN (while projection is a hard way for regularization). 
From Figure~7, we can see that the norms of BN, convolutional layers, and fully connected layers are regularized properly. However, the robust accuracy of the models is still at a low level (basically the same as that without regularization). 
Along with the observation that norm-regularization methods proposed by other authors only \emph{slightly} improve robustness of neural network (as we have discussed in the Introduction), we believe that regularization of norms is ineffective in improving adversarial robustness.

\subsection{Comparison of Norms of Individual Layers}
In the main text, we plot the distribution of the norms of the plain models and adversarially robust models. Here, we compare the norms of the corresponding layers in a model trained with 4 methods, namely, plain (no regularization), ALP, TRADES, PGD-AT. The results are shown in Figure~8-12 (please note that in all the plots, the four bars represent the norms of the plain model, the models trained with ALP, TRADES, and PGD-AT, respectively). The comparison clearly shows that the norms of adversarially robust CNNs are comparable to those of the non-adversarially robust CNNs (plain). Moreover, in RegNet (Figure~9) and SENet (Figure~12), the adversarially robust CNNs even have much larger norms than the non-adversarially robust ones. These results consistently show that large norms do not hurt adversarially robustness. %

\onecolumn



\section*{Supplementary material}

## C  Computing the Set $\mathcal{S}$ and the Gradient

In Lemma 1, we construct the set of indices sets $\mathcal{S}$ in a way that is convenient for the proof. Here, we provide a practical way – a Python function – for computing the set $\mathcal{S}$.

Moreover, we use a slightly different version of momentum for the gradient of norms in our experiments. In Algorithm 1, the momentum is applied as follows: $h \leftarrow \gamma \cdot h + (1-\gamma) \cdot p$. In practice, it would be more convenient and efficient to first decay the historical gradient $h \leftarrow \gamma \cdot h$ and then copy the (sparse) gradient of norms in this step to $h$. This approach is basically the same as the standard one while it updates the gradient of norms slightly faster than the standard one.

```python
def partition(h, w, k, s, p):
    """
    The main function for computine the set S.
    Inputs:
        h - height of the input image
        w - width of the input image
        k - kernel size
        s - stride size
        p - padding size
    Output:
        S - the set of indices set
    """
    if type(k) is not tuple:
        k = (k, k)
    if type(s) is not tuple:
        s = (s, s)
    if type(p) is not tuple:
        p = (p, p)
    r0 = min(k[0], h + 2 * p[0] - k[0] + 1)
    r1 = min(k[1], w + 2 * p[1] - k[1] + 1)

    all_classes = []

    init_classes = equivalence_class(0, 0, (r0, r1), s)
    all_classes += init_classes
    t = (k[0] - r0, k[1] - r1)
    for i in range(0, t[0]+1):
        for j in range(0, t[1]+1):
            if i == 0 and j == 0:
                continue
            classes_new = increment(init_classes, i, j)
            all_classes += classes_new

    idx_set = []
    for classes in all_classes:
        tmp = []
        for c in classes:
            tmp.append(sub2lin(c, k[0], k[1]))
        idx_set.append(tmp)

    # remove redundancy
    idx_set = sorted(idx_set, key=lambda x:len(x))
    idx_set2 = [(c, set(c)) for c in idx_set]
    for i in range(len(idx_set2)):
        c, sc = idx_set2[i]
        for j in range(i+1, len(idx_set2)):
            c_, sc_ = idx_set2[j]
            if sc <= sc_:
                idx_set.remove(c)
                break
```

```python
        return idx_set

    def equivalence_class(x, y, r, s):
        all_idx = []
        for i in range(x, r[0]):
            for j in range(y, r[1]):
                all_idx.append((i,j))

        classes = []
        while len(all_idx) > 0:
            x0, y0 = all_idx[0]
            tmp = []
            for x in range(x0, r[0], s[0]):
                for y in range(y0, r[1], s[1]):
                    tmp.append((x, y))
                    all_idx.remove((x, y))
            classes.append(tmp)

        return classes

    def increment(classes, i, j):
        classes_new = []
        for eqv in classes:
            tmp = [(c[0] + i, c[1] + j) for c in eqv]
            classes_new.append(tmp)
        return classes_new

    def sub2lin(sub, k0, k1):
        return sub[1] * k0 + sub[0]
```



\end{document}